\documentclass{article}

    \usepackage[final, nonatbib]{neurips_2022}

\usepackage[utf8]{inputenc} % allow utf-8 input
\usepackage[T1]{fontenc}    % use 8-bit T1 fonts
\usepackage{url}            % simple URL typesetting
\usepackage{hyperref}       % hyperlinks
\usepackage{booktabs}       % professional-quality tables
\usepackage{amsfonts}       % blackboard math symbols
\usepackage{nicefrac}       % compact symbols for 1/2, etc.
\usepackage{microtype}      % microtypography

\usepackage{epsfig}
\usepackage{graphicx}
\usepackage{amsmath}
\usepackage{amssymb}
\usepackage{mathtools,cuted}
\usepackage[dvipsnames]{xcolor}
\usepackage{helvet}
\usepackage{courier}
\usepackage{bm}
\usepackage{dsfont}
\usepackage{paralist}
\usepackage{algorithm}
\usepackage{algorithmic}
\usepackage{makecell}
\usepackage{capt-of}
\usepackage[english]{babel}
\usepackage{amsthm}

\newtheorem{theorem}{Theorem}[section]

\newtheorem{lemma}[theorem]{Lemma}

\title{Descent Steps of a Relation-Aware Energy Produce Heterogeneous Graph Neural Networks}

\author{%
  Hongjoon Ahn\textsuperscript{\rm 1}\thanks{Work completed during an internship at the AWS Shanghai AI Lab.},\ \ Yongyi Yang\textsuperscript{\rm 2}\footnotemark[1] ,\ \ Quan Gan \textsuperscript{\rm 3},\ \  Taesup Moon \textsuperscript{\rm 1}\thanks{Co-corresponding author.}\ \ and David Wipf \textsuperscript{\rm 3}\footnotemark[2] \\
  \textsuperscript{\rm 1} ECE/IPAI/ASRI/INMC, Seoul National University, 
  \textsuperscript{\rm 2} University of Michigan, 
  \textsuperscript{\rm 3} Amazon Web Services \\
  \texttt{\{hong0805, tsmoon\}@snu.ac.kr}, 
  \texttt{yongyi@umich.edu}, \\
  \texttt{quagan@amazon.com}, 
  \texttt{davidwipf@gmail.com}
}

\newcommand{\mcG}{\mathcal{G}} % Graph
\newcommand{\mcV}{\mathcal{V}} % Set of nodes 

\newcommand{\mcE}{\mathcal{E}} % Set of edges
\newcommand{\mcS}{\mathcal{S}} % Set of node types
\newcommand{\mcT}{\mathcal{T}} % Set of edge types
\newcommand{\mcH}{\mathcal{H}} % Set of compatibility matrices
\newcommand{\mcW}{\mathcal{W}} % Set of weight matrices

\newcommand{\mbY}{\mathbf{Y}} % Unfolding vector
\newcommand{\mbX}{\mathbf{X}} % Input matrix
\newcommand{\mbA}{\mathbf{A}} % Adjacency matrix
\newcommand{\mbL}{\mathbf{L}} % Laplacian matrix
\newcommand{\mbD}{\mathbf{D}} % Degree matrix
\newcommand{\mbP}{\mathbf{P}} % 
\newcommand{\mbQ}{\mathbf{Q}} % 
\newcommand{\mbI}{\mathbf{I}} % Identity matrix
\newcommand{\mbW}{\mathbf{W}} % Weight matrix
\newcommand{\mbH}{\mathbf{H}} % Compatibility matrix
\newcommand{\mbZ}{\mathbf{Z}} % Normalizing constant matrix
\newcommand{\mbC}{\mathbf{C}} 
\newcommand{\mbB}{\mathbf{B}}
\newcommand{\mbM}{\mathbf{M}} 

\newcommand{\bprox}{\textbf{prox}} % proximal operator

\newcommand{\eg}{{\it e.g.}}
\newcommand{\ie}{{\it i.e.}}

\input{def.set}

\begin{document}

\maketitle

\begin{abstract}
%   Heterogeneous graph neural networks (GNNs) achieve great performance on doing node classification or link prediction tasks in semi-supervised learning setting for heterogeneous graphs. However, similar as in the homogeneous GNN case, most of the message passing based heterogeneous GNNs cannot stack large number of GNN layers, and eventually fail to capture the long dependency that exist in graph structured data. To address the above problems, we propose a novel heterogeneous GNN architecture in which GNN layers are derived from the minimization procedure of relation-aware energy function. We show the energy function that firstly introduced in homogeneous graph cannot be directly applied to heterogeneous graph which has multiple node and edge types, and as a result, it is necessary to devise a novel energy function which reflects the relations between two different node types. Experimental results on 8 heterogeneous graphs show that the proposed method achieves competitive node classification accuracy.

Heterogeneous graph neural networks (GNNs) achieve strong performance on node classification tasks in a semi-supervised learning setting. However, as in the simpler homogeneous GNN case, message-passing-based heterogeneous GNNs may struggle to balance between resisting the oversmoothing that may occur in deep models, and capturing long-range dependencies of graph structured data. Moreover, the complexity of this trade-off is compounded in the heterogeneous graph case due to the disparate heterophily relationships between nodes of different types. To address these issues, we propose a novel heterogeneous GNN architecture in which layers are derived from optimization steps that descend a novel relation-aware energy function. The corresponding minimizer is fully differentiable with respect to the energy function parameters, such that bilevel optimization can be applied to effectively learn a functional form whose minimum provides optimal node representations for subsequent classification tasks.  In particular, this methodology allows us to model diverse heterophily relationships between different node types while avoiding oversmoothing effects.  Experimental results on 8 heterogeneous graph benchmarks demonstrates that our proposed method can achieve competitive node classification accuracy. The source code of our algorithm is available at \href{https://github.com/hongjoon0805/HALO}{\textcolor{magenta}{https://github.com/hongjoon0805/HALO}}
\end{abstract}

% \david{I modified the title, still not great, but we can change later if needed.}

% \tsm{I created the openreview submission. I agree we may need to modify the title before the final version.}

\section{Introduction}

Graph structured data, which contains information related to entities as well as the interactions between them, arise in various applications such as social networks or movie recommendations. In a real world scenario, these entities and interactions are often multi-typed, such as \textit{actor}, \textit{director}, \textit{keyword}, and \textit{movie} in a movie review graph \cite{(IMDB)Maas2011} or \textit{authors}, \textit{papers}, \textit{terms}, and \textit{venue} in an academic network \cite{(Academic)Tang2008}.  Hence the relationships between entities can potentially be much more complex than within a traditional homogeneous graph scenario, in which only a single node and edge type exists. 
% Graph neural networks (GNNs) have been shown to be powerful for exploiting the relationship in graph structured data which consists of both the information on the entities as well as interactions between them. In a real world scenario, entities and their interactions are often multi-typed (actor, director, keyword, and movie in a movie review graph; authors, papers, terms, and venue in an academic network \cite{(Academic)Tang2008}), and the relationship between the entities become more complex compared to homogeneous case which contains only a single node and edge type. 
To this end, \textit{heterogeneous} graphs have been proposed as a practical and effective tool to systematically deal with such multi-typed graph-structured data.
% that end, the \textit{heterogeneous} graphs have been proposed as a practical and effective tool to deal with such multi-typed graph-structure data.
With the advent of graph neural networks (GNN) for instantiating deep models that are sensitive to complex relationships between data instances, it is natural that numerous variants have been developed to explicitly handle heterogeneous graphs \cite{cen2019representation,(metapath2vec)Dong2017,(MAGNN)Fu2020,(GPT-GNN)Mu2020,li2021leveraging,(HGB)Lv2021,(R-GCN)Schlichtkrull2018,(HAN)Wang2019,yun2019graph}, with promising performance on downstream tasks such as node classification or link prediction. 

In this vein, message-passing GNNs represent one of the most widely-adopted heterogeneous architectures, whereby a sequence of type-specific graph propagation operators allow neighboring nodes to share and aggregate heterogeneous information \cite{(MAGNN)Fu2020,(R-GCN)Schlichtkrull2018,(HAN)Wang2019}.
% stacking a sequence of propagation layers,
% whereby neighboring nodes share feature information via
% an aggregation function
% multiple GNN layers to aggregate neighborhood features via edge-type dependent aggregate functions
% Several works in heterogeneous graphs are message-passing GNNs which are based on stacking multiple GNN layers that aggregate neighborhood features via edge-type dependent aggregate functions \cite{(MAGNN)Fu2020,(R-GCN)Schlichtkrull2018,(HAN)Wang2019}. 
However, in some cases where long-range dependencies exist between nodes across the graph, stacking too many layers to reflect a larger neighborhood may contribute to a well-known  \textit{oversmoothing} problem, in which node features converge to similar, non-discriminative embeddings \cite{Li2018,(Lose)OonoSuzuki2020}. In the case of homogeneous graphs, this oversmoothing problem has been addressed (among other ways) using modified GNN architectures with layers patterned after the unfolded descent steps of some graph-regularized energy \cite{(APPNP)Gasteiger2019, (TWIRLS)Yang2021}.  Provided minimizers of this energy remain descriminative for the task at hand, arbitrarily deep networks can be trained without risk of oversmoothing, with skip connections between layers organically emerging to favor node-specific embeddings.  However, extending this approach to handle heterogeneous graphs is not straightforward, largely because node/edge heterogeneity (\eg, node-specific numbers of classes or feature dimensions,  multiple types of relationships between nodes) is not easily reflected in the vanilla energy functions adopted thus far for the homogeneous case.  In particular, existing energy functions are mostly predicated on the assumption of graph homophily, whereby neighboring nodes have similar labels.  But, in the heterogeneous case, this assumption is no longer realistic because the labels between nodes of different types may exhibit complex relationships.

% number of classes or feature dimensions for different type of nodes can be different

% due to the multiple types of nodes and edges, the number of classes or feature dimensions for different type of nodes can be different, and edge-type dependent characteristics cannot be easily reflected to the energy function .

% In heterogeneous GNN, the architectures are also based on the message-passing layers, and therefore they suffer from the oversmoothing problem as well. In the case of homogeneous graph, to avoid the oversmoothed representations, 

% \cite{(APPNP)Gasteiger2019, (TWIRLS)Yang2021} proposed an \textit{unfolded} GNN architecture in which the gradient descent steps of graph-regularized energy function serves as propagating multiple GNN layers.

% The interesting property of this view is that it is 

% fundamentally immune to oversmoothing thanks to naturally derived skip connections which is a key solution to overcome this problem \cite{(Deepgcns)Guohao2019}. However, in heterogeneous graph, due to the multiple types of nodes and edges, the number of classes or feature dimensions for different type of nodes can be different, and edge-type dependent characteristics cannot be reflected to this energy function. 
% Therefore, it is hard to directly apply the energy function to the heterogeneous graph. 

To this end, we first introduce a novel heterogeneous GNN architecture with layers produced by the minimization of a relation-aware energy function. This energy produces regularized heterogeneous node embeddings, relying on relation-dependent compatibility matrices as have been previously adopted in the context of belief propagation on heterogeneous graphs, e.g., ZooBP \cite{(ZooBP)Eswaran2017}. Secondly, we derive explicit propagation rules for our model layers by computing gradient descent steps along the proposed energy, with quantifiable convergence conditions related to the learning rate being considered as well. Furthermore, we provide interpretations of each component of the resulting propagation rule, which consists of skip-connections, linearly transformed feature aggregation, and a self-loop transformation. 
% \tsm{add more contributions? like interpretation of our rule / convergence condition / using nonlinearity, etc.}
With respect to experiments, our contributions are twofold. First, we show that our algorithm outperforms state-of-the-art baselines on various benchmark datasets. And secondly, we analyze the the effectiveness of using the compatibility matrix and the behavior of the unfolding step as the number of propagations increases.

% Starting from the energy function firstly proposed in the homogeneous graphs, we show that directly applying this to heterogeneous graph cannot represent the relation-aware regularization, and due to the dimension or label mismatch between two different node types, the regularization in the energy function should be compatible with different node types.

% Based on this viewpoint, we fist propose an energy function expressed by relation dependent compatibility matrices introduced in a belief propagation algorithm, ZooBP\cite{(ZooBP)Eswaran2017}, on heterogeneous graph. 

% Second, we derive a propagation rule of GNN layer from the gradient descent step on minimizing the proposed energy function. In the experiments, our contributions are two fold. First, we show that our algorithm outperforms state-of-the-art baselines in various benchmark datasets. Second, we give analyses on the effectiveness of using the compatibility matrix and the behavior of the unfolding step as the number of propagations increases.

% \noindent\textbf{Related work} \ \
% Various approaches for message-passing GNNs in heterogeneous graph have been proposed
% \cite{(R-GCN)Schlichtkrull2018,(MAGNN)Fu2020,(HGT)Hu2020,(HAN)Wang2019,(HGB)Lv2021}, 

\vspace{-.2cm}
\section{Existing Homogeneous GNNs from Unfolded Optimization} \label{sec:homogeneous_unfolding}
\vspace{-.2cm}
% \david{Need to make sure the notation here is consistent with later sections, and not overly redundant.}
Consider a \textit{homogeneous} graph $\calG = \{\calV,\calE\}$, with $n = |\calV|$ nodes and $m=|\calE|$ edges.  We define $\mbL \in \mathbb{R}^{n\times n}$ as the Laplacian of $\calG$, meaning $\mbL = \mbD-\mbA$, in which $\mbD$ and $\mbA$ are degree and adjacency matrices, respectively.  Unfolded GNNs incorporate graph structure by attaining optimized embeddings $\mbY^* \in \mathbb{R}^{n\times d}$ that are functions of some adjustable weights $\mbW$, i.e., $\mbY^* \equiv \mbY^*(\mbW)$,\footnote{For brevity, we will frequently omit this dependency on $\mbW$ when the context is clear.} where $\partial \mbY^*(\mbW)/ \partial \mbW$ is computable by design. These embeddings are then inserted within an application-specific loss 
\vspace*{-0.0cm}
\begin{equation} \label{eq:meta_loss}
\ell_{\btheta}\left(\btheta,\mbW \right) \triangleq \sum_{i=1}^{n'} \calD\big( g\left[  \bm y_i^*(\mbW); \btheta \right], \bt_i  \big),
\end{equation}
\vspace*{-0.0cm}
in which $g : \mathbb{R}^d \rightarrow \mathbb{R}^c$ is some differentiable node-wise function with parameters $\btheta$ and $c$-dimensional output tasked with predicting ground-truth node-wise targets $\bt_i \in \mathbb{R}^c$.  Additionally, $\bm y_i^*(\mbW)$ is the $i$-th row of $\mbY^*(\mbW)$,  $n' < n$ is the number of labeled nodes (we assume w.l.o.g.~that the first $n'$ nodes are labeled), 
and $\calD$ is a discriminator function, \eg, cross-entropy for node classification, as will be our focus. Given that $\mbY^*(\mbW)$ is differentiable by construction, we can optimize $\ell_{\btheta}\left(\btheta,\mbW \right)$ via gradient descent (over \textit{both} $\btheta$ and $\mbW$ if desired) to obtain our final predictive model.

At a conceptual level, what differentiates GNNs inspired by unfolded optimization is how $\mbY^*(\mbW)$ is obtained.  Specifically, these node embeddings are chosen to be the minimum of a lower-level, graph-regularized energy function \cite{chen2021does,ma2020unified,pan2021a,(TWIRLS)Yang2021,thatpaper,zhu2021interpreting}.  Orignally inspired by \cite{Zhou2003}, the most common selection for this energy is given by
\begin{equation} \label{eq:basic_objective}
\ell_{\mbY}(\mbY) \triangleq  \left\|\mbY - f\left(\mbX ; \mbW  \right) \right\|_{\calF}^2 + \lambda \mbox{tr}\left[\mbY^\top \mbL \mbY  \right],
\end{equation}
where $\lambda > 0$ is a trade-off parameter, $\mbY \in \mathbb{R}^{n\times d}$ is a matrix of trainable embeddings, meaning $d$-dimensional features across $n$ nodes, and $f\left(\mbX;\mbW\right)$ is a base model parameterized by $\mbW$ that produces initial target embeddings using the $d_0$-dimensional node features $\mbX \in \mathbb{R}^{n\times d_0}$.  Minimization of (\ref{eq:basic_objective}) over $\mbY$ results in new node embeddings $\mbY^*$ characterized by a balance between similarity with  $f\left(\mbX ; \mbW  \right)$ and smoothness across graph edges as favored by the trace term $\mbox{tr}\left[\mbY^\top \mbL \mbY \right] =  \sum_{\{i,j\} \in \calE}\left\| \bm y_i - \bm y_j \right\|_2^2$, where $\bm y_i$ is $i$-th row of $\mbY$.  While (\ref{eq:basic_objective}) can be solved in a closed-form as
\begin{equation} \label{eq:closed form}
\mbY^*(\mbW) = \arg\min_{\mbY} \ell_\mbY(\mbY ) = \left(\mbI + \lambda \mbL \right)^{-1} f(\mbY;\mbW),
\end{equation}
for large graphs this analytical solution is not practically computable.  Instead, we may form initial node embeddings $\mbY^{(0)} = f\left(\mbX ; \mbW  \right)$, and then apply gradient descent to converge towards the minimum, with the $k$-th iteration given by
\begin{equation} \label{eq:basic_grad_step}
\mbY^{(k)} = \mbY^{(k-1)} - \alpha\left[ \left( \lambda \mbL  + \mbI\right) \mbY^{(k-1)} - f\left(\mbX ; \mbW  \right) \right], 
\end{equation}
in which $\frac{\alpha}{2}$ is the step size.  Given that $\mbL$ is generally sparse, computation of (\ref{eq:basic_grad_step}) can leverage efficient sparse matrix multiplications.  Combined with the fact that the loss is strongly convex, and provided that $\alpha$ is chosen suitably small, we are then guaranteed to converge towards the unique global optima with efficient gradient steps, meaning that for some $K$ sufficiently large $\mbY^{(K)}(\mbW) \approx \mbY^*(\mbW)$.

Critically, $\mbY^{(K)}(\mbW)$ remains differentiable with respect to $\mbW$, such that we may substitute this value into (\ref{eq:meta_loss}) in place of $\mbY^*(\mbW)$ to obtain an entire pipeline that is end-to-end differentiable w.r.t.~$\mbW$.  Moreover, per the analyses from \cite{ma2020unified,pan2021a,(TWIRLS)Yang2021,thatpaper,zhu2021interpreting}, the $k$-th \textit{unfolded} iteration of (\ref{eq:basic_grad_step}) can be interpreted as an efficient form of GNN layer.  In fact, for certain choices of the step size, and with the incorporation of gradient pre-conditioning and other reparameterization factors, these iterations can be exactly reduced to popular canonical GNN layers such as those used by GCN \cite{(GCN)Kipf2017}, APPNP \cite{(APPNP)Gasteiger2019}, and others.
However, under broader settings, we obtain unique GNN layers that are naturally immune to oversmoothing, while nonetheless facilitating long-range signal propagation across the graph by virtue of the anchoring effect of the underlying energy function.  In other words, regardless of how large $K$ is, these iterations/layers will converge to node embeddings that adhere to the design criteria of (\ref{eq:basic_objective}).  This is the distinct appeal of GNN layers motivated by unfolded optimization, at least thus far in the case of homogeneous graphs.

\vspace{-.2cm}
\section{New Heterogeneous GNN Layers via Unfolding}
\vspace{-.2cm}

In this section, we explore novel extensions of unfolded optimization to handle heterogeneous graphs, which, as mentioned previously, are commonly encountered in numerous practical application domains.  After introducing our adopted notation, we first describe two relatively simple, intuitive attempts to accommodate heterogeneous node and edge types.  We then point out the shortcomings of these models, which are subsequently addressed by our main proposal: a general-purpose energy function and attendant proximal gradient steps that both (i) descend the aforementioned energy, and (ii) in doing so facilitate flexible node- and edge-type dependent message passing with non-linear activations where needed.  The result is a complete, interpretable heterogeneous GNN (HGNN) architecture with layers in one-to-one correspondence with the minimization of a heterogeneous graph-regularized energy function.

% We introduce a novel heterogeneous GNN layer that effectively handles the oversmoothing problem and can be equipped with non-linear activations and layer independent weights. In the next subsection, we show that we can achieve this by applying proximal gradient descent to energy function that consists of regularizations on node embeddings.

Before we begin, we introduce some additional notation.  An undirected heterogeneous graph $\mcG=(\mcV, \mcE)$ is a collection of node types $\mcS$ and edge types $\mcT$ such that $\mathcal{V} = \bigcup_{s\in\mcS} \mcV_s$ and $\mcE = \bigcup_{t\in\mcT} \mcE_t$.  Here $\mcV_s$ denotes a set of $n_s = |\calV_s|$ nodes of type $s$, while $\mcE_t$ represents a set of edge type $t$. Furthermore, we use $\mcT_{ss'}$ to refer to the set of edge types connecting node types $s$ and $s'$. Note that $\mcT$ includes both the canonical direction, $t \in \mcT_{ss'}$, and the inverse direction, $t_{\text{inv}} \in \mcT_{s's}$, that corresponds with a type $t$ edge.\footnote{For example, if $t\in\mcT_{ss'}$ is ("paper"-"author"), then $t_{\text{inv}}\in \mcT_{s's}$ is ("author"-"paper")}

\vspace{-.2cm}
\subsection{First Attempts at Heterogeneous Unfolding}
\vspace{-.2cm}

A natural starting point for extending existing homogeneous models to the heterogeneous case would be to adopt the modified energy
\begin{equation} \label{eq:first_simple_hetero_energy}
    \ell_{\mbY}(\mbY) \triangleq  \sum_{s\in\mcS}\left[\frac{1}{2}\left\|\mbY_s-f(\mbX_s;\mbW_s)\right\|_{\mathcal{F}}^2 + \frac{\lambda}{2}\sum_{s'\in\mcS}\sum_{t\in\mcT_{ss'}} \mbox{tr}\left(\mbY^\top \mbL_t \mbY\right) \right], 
\end{equation}
in which $\mbL_t$ is the graph Laplacian involving all edges of relation type $t$ and $\mbW_s$ parameterizes a type-specific transformation of node representations. Here, we have simply introduced type-specific transformations of the initial node representations, as well as adding a separate graph regularization term for each relation type $t$.  However, this formulation is actually quite limited as it can be reduced to the equivalent energy 
\begin{equation} \label{eq:energy_simplification}
    \ell_{\mbY}(\mbY) \equiv  \frac{1}{2}\left\|\mbY-\widetilde{f}(\widetilde{\mbX};\mcW)\right\|_{\calF}^2 + \frac{\lambda}{2}  \mbox{tr}\left(\mbY^\top \mbC \mbY\right),~~\mbox{with}~~ \mbC = \sum_{s\in\mcS}\sum_{s'\in\mcS}\sum_{t\in\mcT_{ss'}} \mbL_t,
\end{equation}
in which $\widetilde{\mbX}$ represents the original node features concatenated with the node type, and $\widetilde{f}$ is some function of the augmented features parameterized by $\mcW \triangleq \{\mbW_s\}_{s\in\mcS}$.  Note that by construction, $\mbC$ will necessarily be positive semi-definite, and hence we may conclude through  (\ref{eq:energy_simplification}) that  (\ref{eq:first_simple_hetero_energy}) defaults to the form of a standard quadratically-regularized loss, analogous to (\ref{eq:basic_objective}), that has often been applied to graph signal processing problems \cite{ioannidis2018}.  Therefore the gradient descent iterations of this objective will closely mirror the existing homogeneous unfolded GNN architectures described in Section \ref{sec:homogeneous_unfolding}, and fail to capture the nuances of heterogeneous data.

To break the symmetry that collapses all relation-specific regularization factors into $\mbox{tr}\left(\mbY^\top \mbC \mbY\right)$, a natural option is to introduce $t$-dependent weights as in
\begin{equation} \label{eq:second_simple_hetero_energy}
    \ell_{\mbY}(\mbY) \triangleq  \sum_{s\in\mcS}\left[\frac{1}{2}\left\|\mbY_s-f(\mbX_s;\mbW_s)\right\|_{\mathcal{F}}^2 + \frac{\lambda}{2}\sum_{s'\in\mcS}\sum_{t\in\mcT_{ss'}} \mbox{tr}\left(\mbY^\top \mbL_t \mbY \mbM_t \right) \right], 
\end{equation}
in which the set $\{\mbM_t \}$ is trainable.
% \tsm{Are $\mbW_t$ the same as $\mbW_s$? If not, the notation could be a bit confusing..}
While certainly more expressive than  (\ref{eq:first_simple_hetero_energy}), this revision is nonetheless saddled with several key limitations:  (i) Unless additional constraints are included on $\{\mbM_t\}$ (e.g., PSD, etc.),  (\ref{eq:second_simple_hetero_energy}) may be unbounded from below; (ii) While $\mbM_t$ may be asymmetric, w.l.o.g.~the penalty can be equivalently expressed with symmetric weights, and hence the true degrees of freedom are limited;\footnote{Note that $\mbM_t = \frac{1}{2}[(\mbM_t+\mbM_t^{\top})+(\mbM_t-\mbM_t^{\top})]$ for any weight matrix $\mbM_t$. Since $\bm y^{\top}(\mbM_t-\mbM_t^{\top})\bm y=0$ for any $\bm y$ and $\mbM_t$, the penalty actually only depends on the symmetric part $(\mbM_t+\mbM_t^{\top})$.} (iii) This model could be prone to overfitting,  since during training it could be that $\mbM_t \rightarrow 0$ in which case the graph-based regularization is turned off;\footnote{For example, suppose the node features of training nodes are highly correlated with the labels, or actually are the labels (as in some prior label propagation work). Then the model could just learn $\mbM_t = 0$ and achieve perfect reconstruction on the training nodes to produce zero energy, and yet overfit since graph propagation is effectively turned off.} and (iv) Because  $\mbox{tr}\left[\mbY^\top \mbL_t \mbY  \mbM_t \right] = \sum_{(i,j) \in \calE_t} \left( \bm y_i - \bm y_j \right)^\top \mbM_t \left(\bm y_i - \bm y_j \right)$, the energy function (\ref{eq:second_simple_hetero_energy}) is symmetric w.r.t.~the order of the nodes in the regularization term, and therefore any derived message passing must also be symmetric, \ie, we cannot exploit asymmetric penalization aligned with heterogeneous relationships in the data.

Given then that both  (\ref{eq:first_simple_hetero_energy}) and  (\ref{eq:second_simple_hetero_energy}) have notable shortcomings, it behooves us to consider revised criteria for selecting a suitable heterogeneous energy function.  We explore such issues next. 

\vspace{-.2cm}
\subsection{A More Expressive Alternative Energy Function}
\vspace{-.2cm}

As an entry point for developing a more flexible class of energy functions that is sensitive to nuanced relationships between different node types, we consider previous work developing various flavors of both label and belief propagation \cite{(ZooBP)Eswaran2017,pearl2022reverend,(CAMLP)Yuto2016,Zhou2003}.  In the more straightforward setting of homogeneous graphs under homophily conditions, label propagation serves as a simple iterative process for smoothing known training labels across edges to unlabeled nodes in such a way that, for a given node $i$, the predicted node label $\hat{\bt}_i$ will approximately match the labels of neighboring nodes $\calN_i$, \ie, $\hat{\bt}_i \approx \hat{\bt}_j$ when $j \in  \calN_i$.  An analogous relationship holds when labels are replaced with beliefs \cite{pearl2022reverend}.  

However, in broader regimes with varying degrees of heterophily, it is no longer reasonable for such a simple relationship to hold, as neighboring nodes may be more inclined to have \textit{dissimilar} labels.  To address this mismatch, a compatibility matrix $\mbH$ must be introduced such that now label propagation serves to instantiate $\hat{\bt}_i \approx \hat{\bt}_j \mbH $ for $j \in \calN_i$ \cite{(CAMLP)Yuto2016} (here we treat each $\hat{\bt}_j$ as a row vector).  If each predicted label $\hat{\bt}_i$ is approximately a one-hot vector associated with class membership probabilities, then the $(k,l)$-th element of $\mbH$ roughly determines how a node of class $k$ influences neighbors of class $l$.  And if we further extend to heterogeneous graphs as has been done with various forms of belief propagation \cite{(ZooBP)Eswaran2017}, it is natural to maintain a unique (possibly non-square) compatibility matrix $\mbH_t$ for each relation type $t$, with the goal of finding predicted labels (or beliefs) of each node type satisfying $\hat{\bt}_{si} \approx   \hat{\bt}_{s'j} \mbH_t$ for $s,s' \in \calS$, $t \in \calT_{ss'}$, and $(i,j) \in \calE_t$. From a practical standpoint, each such $\mbH_t$ can be estimated from the data using the statistics of the labels (or beliefs) of nodes sharing edges of type $t$.

Returning to our original goal, we can use similar intuitions to guide the design of an appropriate regularization factor for learning heterogenerous graph node representations (as required by HGNNs). Specifically, given $s,s' \in \calS$, $t \in \calT_{ss'}$, and $(i,j) \in \calE_t$, we seek to enforce $\bm y_{si} \mbH_t \approx \bm y_{s'j}$, which then naturally motivates the energy
\begin{align}
    \ell_{\mbY}(\mbY) \triangleq & \sum_{s\in\mcS}\left[\frac{1}{2}||\mbY_s-f(\mbX_s;\mbW_s)||_{\mathcal{F}}^2 + \frac{\lambda}{2}\sum_{s'\in\mcS}\sum_{t\in\mcT_{ss'}}\sum_{(i,j)\in\mcE_t}||\bm y_{si}\mbH_t-\bm y_{s'j}||_2^2 \right], \label{eq:energy}
\end{align}
where $\lambda$ denotes a trade-off parameter, $\mbX_s \in \mathbb{R}^{n_s \times d_{0s}}$ represents $d_{0s}$-dimensional initial features and $\mbY_s \in \mathbb{R}^{n_s \times d_s}$ is the embedding of $d_s$-dimensional features on $n_s$ nodes for node type $s$. Note that $\bm y_{si}$ and $\bm y_{s'j}$ are the $i$-th and $j$-th row of $\mbY_s$ and $\mbY_{s'}$, respectively. Additionally, $\mbH_t \in \mathbb{R}^{d_s \times d_{s'}}$ denotes a compatibility matrix that matches the dimension of two different node embeddings (\ie, $\bm y_{si}$ and $\bm y_{s'j}$) via linear transformation for edge type $t$.
% , and $f(\mbX_s;\mbW)$ denotes linear function or neural network with parameter $\mbW$ generating the embedding of $d_{0s}$-dimensional initial feature $\mbX_s \in \mathbb{R}^{n_s \times d_{0s}}$. 

The energy function (\ref{eq:energy}) exhibits the following two characteristics:
\begin{itemize}
    \item The first term prefers that the embedding $\mbY_s$ should be close to that from $f(\mbX_s;\mbW_s)$.
    \item The second term prefers that for two nodes of type $s$ and $s'$ connection by an edge of type $t$, the embeddings $\bm y_{si} \mbH_t$ and $\bm y_{s'j}$ should be relatively close to one another.
\end{itemize}
Even when the dimensions of two embeddings are the same (\ie, $d_s = d_{s'}$), $\mbH_t$ still serves an important role in allowing the embeddings to lie within different subspaces to obtain compatibility. 

\vspace{-.1cm}
\subsection{Analysis and Derivation of Corresponding Descent Steps}
\vspace{-.1cm}

The objective (\ref{eq:energy}) can be written as a matrix form as follows:
\begin{align}
    \ell_{\mbY}(\mbY) &= \sum_{s\in\mcS} \Bigg[\frac{1}{2} ||\mbY_s-f(\mbX_s;\mbW_s)||_{\mathcal{F}}^2 \nonumber\\
    &+ \frac{\lambda}{2} \sum_{s'\in\mcS}\sum_{t\in\mcT_{ss'}} \text{tr}((\mbY_s\mbH_t)^{\top}\mbD_{st}(\mbY_s\mbH_t)-2(\mbY_s\mbH_t)^{\top}\mbA_{t}\mbY_{s'}+\mbY_{s'}^{\top}\mbD_{s't_{\text{inv}}}\mbY_{s'})\Bigg], \label{eq:energy_matrix}
\end{align}
where $\mbA_t$ denotes the adjacency matrix for an edge type $t$, and $\mbD_{st}$ denotes the type-$t$ degree matrix of type-$s$ nodes. 

The closed-form optimal point of  (\ref{eq:energy_matrix}) is provided by the following result; see the Supplementary Materials for the proof.
\begin{lemma}\label{lemma:closed-form} The unique solution $\mbY^* (\mcW, \mcH)$ minimizing  (\ref{eq:energy_matrix}) satisfies
\begin{align}
    vec(\mbY^* (\mcW, \mcH)) = (\mbI + \lambda(\mbQ-\mbP+\mbD))^{-1}vec(\widetilde{f}(\widetilde{\mbX};\mcW)), \label{eq:closed_form}
\end{align}
where $vec(\mbB)=[\mathbf{b}_1^{\top}, ..., \mathbf{b}_n^{\top}]^{\top}$ for matrix $\mbB$, $\mcH \triangleq \{\mbH_t\}_{t\in\mcT} $ denotes the set of all compatibility matrices, and we have defined the matrices $\mbP$, $\mbQ$, and $\mbD$ as
\begin{align}
    \mbP = \left[ \begin{array}{ccc}
        \mbP_{11} & ... & \mbP_{1|\mcS|}  \\
        ... & ... & ...  \\
        \mbP_{|\mcS|1} & ... & \mbP_{|\mcS||\mcS|} 
    \end{array} \right] ; \mbP_{ss'} = \sum_{t\in\mcT_{ss'}} ((\mbH_t^{\top}+\mbH_{t_{\text{inv}}}) \otimes \mbA_t) \nonumber
\end{align}

\begin{align}
    \mbQ = \bigoplus_{s\in\mcS} \mbQ_s; \mbQ_s=\sum_{s'\in\mcS}\sum_{t\in\mcT_{ss'}} (\mbH_t\mbH_t^{\top}\otimes\mbD_{st}), \mbD = \bigoplus_{s\in\mcS}\mbI\otimes\mbD_s. \nonumber
\end{align}
Here $\otimes$ denotes the Kronecker product, $\bigoplus_{i=1}^n \mbA_i = diag(\mbA_1,...,\mbA_n)$ denotes a direct sum of $n$ square matrices $\mbA_1,...,\mbA_n$, and $\mbD_s \triangleq \sum_{s'\in\mcS}\sum_{t\in\mcT_{ss'}} \mbD_{s't}$ represent a sum of degree matrices over all node types $s' \in \mcS$ and all edge types $t \in \mcT_{ss'}$.
\end{lemma}

We can interpret this solution as transforming $\widetilde{f}(\widetilde{\mbX};\mcW)$ into an embedding that has both local and global information of the graph structure \cite{Zhou2003}. Therefore, this can be treated as an appropriate graph-aware embedding for carrying out a downstream task such as node classification.

% \david{I think this paragraph could be moved to a separate subsection near the end of Section 3, perhaps combined with Algorithm 1.  Note that Section 3.4 is already quite long, and this content is about the final bilevel optimization process, not about minimizing the lower-level energy function.} 

However, for practically-sized graph data, computing the inverse $(\mbI + \lambda(\mbQ - \mbP + \mbD))^{-1}$ in  (\ref{eq:closed_form}) is prohibitively expensive. To resolve this issue, similar to the homogeneous case we instead apply gradient descent to  (\ref{eq:energy_matrix}) over $\mbY_s$ for all node types $s \in \mcS$ to approximate $\mbY^* (\mcW, \mcH)$. Starting from the initial point $\mbY_s^{(0)}=f(\mbX_s;\mbW_s)$ for each $s \in \mcS$, several steps of gradient descent are performed to obtain an appropriate embedding for carrying out a downstream task. Since all the intermediate steps of gradient descent are differentiable features w.r.t $\mcW$ and $\mcH$, we can just plug-in the result of performing $K$ iterations of gradient descent, $\mbY^{(K)}_s$, into the  (\ref{eq:loss_fn}). 
% \david{Seems some $k$ should be $K$, where $K$ is the total number of gradient steps, not the generic index for any step.  Maybe need to resolve here and below.}
% In this procedure, $\mbY^{(k)}_s$ can be viewed as a feature of trainable layer in GNN architecture, and we call this gradient descent steps as \textit{unfolding}\david{This phrasing makes it sound like we are just now introducing the notion of unfolding; but unfolding is repeatedly mentioned earlier.} \cite{GregorLeCun2010,Hershey2014,Sprechmann2015,(TWIRLS)Yang2021}.

In order to get $\mbY^{(K)}_s$ using gradient descent, we first compute the gradient of  (\ref{eq:energy_matrix}) w.r.t $\mbY_s$:
\begin{align}
    \nabla_{\mbY_s}\ell_{\mbY}(\mbY) = (\mbI +\lambda\mbD_s)\mbY_s-f(\mbX_s;\mbW_s) + \lambda \sum_{s'\in\mcS}\sum_{t\in\mcT_{ss'}} \big(\mbD_{st}\mbY_s(\mbH_t\mbH_t^{\top})-\mbA_t\mbY_{s'}(\mbH_t^{\top}+\mbH_{t_{\text{inv}}})\big).\label{eq:gradient}
\end{align}
Note, if the condition number of  (\ref{eq:gradient}) is large, the convergence speed of gradient descent can be slow \cite{NocedalWright2006}. To reduce the convergence time effectively, we use the Jacobi preconditioning technique \cite{Axelsson1996}, rescaling the gradient step using $\widetilde{\mbD}_s^{-1}\triangleq (\mbI+\lambda\mbD_s)^{-1}$. Then iteration $k+1$ of gradient descent on $\mbY_s$ can be computed as
\begin{align}
    &\mbY_s^{(k+1)}= \nonumber\\
    &\underbrace{(1-\alpha)\mbY_s^{(k)}}_{(a)} + \alpha \widetilde{\mbD}_s^{-1} \big[\underbrace{f(\mbX_s;\mbW_s)}_{(b)} + \lambda \sum_{s'\in\mcS}\sum_{t\in\mcT_{ss'}} \big(\underbrace{\mbA_t\mbY_{s'}^{(k)}(\mbH_t^{\top}+\mbH_{t_{\text{inv}}})}_{(c)} - \underbrace{\mbD_{st}\mbY_s^{(k)}(\mbH_t\mbH_t^{\top})}_{(d)}\big)\big], \label{eq:unfolding_precond}
\end{align}
where $\alpha$ denotes a step size. The above expression can be considered as a forward propagation rule for the $k$-th layer in a GNN model. In  (\ref{eq:unfolding_precond}), the terms (a), (b), (c), and (d) can be interpreted as follows:

\noindent\textbf{Terms (a) \& (b)} \ \
Each term can be treated as a skip connection from the previous layer $\mbY_s^{(k)}$ and the input features $f(\mbX_s;\mbW_s)$, respectively. Unlike some previous work \cite{(Deepgcns)Guohao2019,(HGB)Lv2021} that adopted skip-connections as a heuristic solution to resolve the oversmoothing problem \cite{Li2018,(Lose)OonoSuzuki2020}, similar to unfolded GNNs in the homogeneous case \cite{(APPNP)Gasteiger2019,(TWIRLS)Yang2021}, the skip connections in our expression are naturally derived from the gradient descent step for minimizing the energy function.
% \david{Should we add a reference that analogous skip connections occur in the homogeneous unfolding case?  Otherwise a reviewer could think we overclaim or miss references.}

\noindent\textbf{Term (c)} \ \
This term accumulates the propagated feature vectors of neighboring nodes. The matrix $\mbH_t^{\top} + \mbH_{t_{\text{inv}}}$ serves as edge-type specific transformations in a bi-directional way. Note that all matrices in $\mcH$ are shared across all propagation layers, and as a result, the number of parameters does not increase as the propagation proceeds.

\noindent\textbf{Term (d)} \ \
This term is a self-loop transformation that depends on the edge type $t$. The matrix $\mbH_t\mbH_t^{\top}$ introduces an edge-type specific transformation in a uni-directional way, and the matrix $\mbD_{st}$ reflects the strength of the self-loop transformation relying on edge type $t$. Note that, unlike prior work \cite{(HGB)Lv2021, (R-GCN)Schlichtkrull2018}, we use edge-type specific degree and transformation matrices, and these are also shared across all GNN layers.

We reiterate that all components of (\ref{eq:unfolding_precond}) have naturally emerged from the gradient descent procedure for minimizing the underlying energy function.  And in terms of convergence, we have the following result; again, the proof is deferred to the Supplementary Materials:
% we address how executing the stated unfolding steps  with an appropriate $\alpha$ produces a global minimum of  (\ref{eq:energy_matrix}), $\mbY^* (\mcW, \mcH)$. (Note that if we choose a relatively large $\alpha$ the energy function value may diverge such that the graph structure is not properly reflected.)  To this end, we have the following, with the proof deferred to the Supplementary Materials.
% \david{Should specify exactly what is meant by convergence.}
% \david{Can we include the convergence conditions, e.g., required step size needed for guaranteed convergence, etc.}
% \david{This content still doesn't say what convergence means, e.g., converges to what? Local minimum, global minimum, stationary point, fixed point, etc.? I guess is should be the global minimum since I think the objective is strongly convex, but if so, we need to say this explicitly.}

\begin{theorem} \label{theorem:convergence}
The iterations (\ref{eq:unfolding_precond}) are guaranteed to monotonically converge to the unique global minimum of (\ref{eq:energy}) provided that
\begin{align}
    \alpha < \frac{2 + 2\lambda d_{\text{min}}}{1 + \lambda (d_{\text{min}} + \sigma_{\text{max}})}, \label{eq:convergence}
\end{align}
where $d_{\text{min}}$ is the minimum diagonal element of $\mbD$
% \david{What exactly does $\mbD_{ii}$ refer to?} \hong{This refers to the diagonal elements of matrix $\mbD$. $\mbD$ is defined in our Lemma}\david{This is not clear because we have already used $\mbD$ with other subscripts to refer to completely different quantities, like in equation 11 or 12.  Let me reword this.} 
and $\sigma_{\text{max}}$ is a maximum eigenvalue of matrix $(\mbQ - \mbP)$.
\end{theorem}

% \begin{theorem} \label{theorem:convergence} The sufficient condition for convergence of the iterative updates in  (\ref{eq:unfolding_precond}) with respect to $\alpha$ is:
% \begin{align}
%     \alpha < \frac{2 + 2\lambda d_{\text{min}}}{1 + \lambda (d_{\text{min}} + \sigma_{\text{max}})}, \label{eq:convergence}
% \end{align}
% where $d_{\text{min}} \triangleq \min_i \mbD_{ii}$ is the minimum degree, and $\sigma_{\text{max}}$ is a maximum eigenvalue of matrix $(\mbQ - \mbP)$.
% \end{theorem}

% \yyy{I wrote a short analysis of the time complexity here, you may move it to a suitable place. }
\noindent\textbf{Time Complexity} \ \ Although seemingly complicated, the time complexity of executing  (\ref{eq:unfolding_precond}) once is $O(md + nd^2)$, where $d=\max\{d_s|s\in \mathcal S\}$. In this way the complexity is on the same level as R-GCN \cite{(R-GCN)Schlichtkrull2018}, a commonly-used heterogeneous GNN variant, and it can be efficiently implemented through off-the-shelf sparse matrix libraries.

\vspace{-.2cm}
\subsection{Generalization to Nonlinear Activations}
\vspace{-.2cm}
% \david{Were the nonlinear activations actually used in the experiments?  Note that for some homogeneous cases from other projects, we have found that they don't always help.} \hong{We always use ReLU activation in our experiments. It can make the energy function to be converged very fast.} \david{Ok, but it might be worth trying without ReLU, maybe for one of the datasets where the performance is not as good.  Because removing the ReLU might increase the accuracy.} 
In various message-passing GNNs \cite{(G-SAGE)Hamilton2017,(GCN)Kipf2017,(HGB)Lv2021,(R-GCN)Schlichtkrull2018}, node-wise ReLU activations are applied to the intermediate features during the forward propagation step. In our setting, by adding additional regularization or constraints to our original energy function and using proximal gradient descent methods \cite{(Proximal)Parikh2014}, such nonlinear activations can also be naturally derived without any heuristic modifications on our propagation step. Similar ideas have been applied to simpler homogeneous GNN models \cite{(TWIRLS)Yang2021}.
% \david{I think we should add a reference here, and not create the impression that we are the first to consider proximal methods with GNNs.  Can probably just cite one reference like TWIRLS, that would be sufficient.}
% Including the regularization or constraints to  (\ref{eq:energy}) exactly corresponds to applying nonlinear activation to the propagation step (\ref{eq:unfolding_precond}). 
More specifically, let $\phi_s:\mathbb{R}^{d_s}\rightarrow\mathbb{R}$ be an arbitrary convex function of node embeddings. Then our optimization objective can be modified the minimization of
\vspace{-.08in}
\begin{align}
    \ell_{\mbY}(\mbY) + \sum_{s\in\mcS}\sum_{i=1}^{n_s} \phi_s (\bm y_{si}). \label{eq:energy_revised}
\end{align}
% \vspace{-.02in}
Instead of directly minimizing  (\ref{eq:energy_revised}) via vanilla gradient descent, we employ the proximal gradient descent (PGD) method \cite{(Proximal)Parikh2014} as follows.
% \david{We should probably include some references for this, since proximal methods have been used for similar purposes before.}. 
We first denote the relevant proximal operator as
\begin{align}
    \bprox_{\phi}(\bv) \triangleq \arg \min_{\bm y} \frac{1}{2} ||\bv - \bm y||_2^2 + \phi(\bm y). \label{eq:prox_operator}
\end{align}
Then PGD iteratively minimizes  (\ref{eq:energy_revised}) by computing 
\begin{align}
    \bar{\mbY}_s^{(k+1)} &:= \mbY_s^{(k)} - \alpha \widetilde{\mbD}_s^{-1}\nabla_{\mbY_s^{(k)}}\ell_{\mbY}(\mbY) \label{eq:before_prox} \\
    \mbY_s^{(k+1)} &:= \bprox_{\phi_s}(\bar{\mbY}_s^{(k+1)}). \label{eq:after_prox}
\end{align}
For our purposes, we set $\phi_s (\bm y)\triangleq \sum_{i=1}^{d_s} \mathcal{I}_{\infty}[y_{i}<0]$, in which $\mathcal{I}_{\infty}$ is an indicator function that assigns an infinite penalty to any $y_{i}<0$.  In this way, the corresponding proximal operator becomes $\bprox_{\phi_s}(\bv)=\text{ReLU}(\bv)=\text{max}(\textbf{0}, \bv)$, in which the maximum is applied to each dimension independently.

% In  (\ref{eq:after_prox}), we apply the proximal operator to each node representation $\bm y_{si}$ independently. If we set $\phi_s (\bm y)\triangleq \sum_{j=1}^{d_s} \mathcal{I}_{\infty}[y_{sj}<0]$, where $\mathcal{I}_{\infty}$ is an indicator function that assigns an infinite penalty to any $y_i<0$, the corresponding proximal operator satisfies $\bprox_{\phi}(\bv)=\text{ReLU}(\bv)=\text{max}(\textbf{0}, \bv)$.
\vspace{-.2cm}
\subsection{The Overall Algorithm}
\vspace{-.2cm}
After performing $K$ iterations of  (\ref{eq:unfolding_precond}) with the proximal operator in  (\ref{eq:after_prox}), we obtain $\mbY^{(K)}(\mcW, \mcH)$, and analogous to (\ref{eq:meta_loss}), the meta-loss function for downstream tasks becomes
% Using the differentiable embedding $\mbY^* (\mbW, \mcH)$,
\begin{align}
    \ell_{\Theta}(\Theta, \mcW, \mcH) = \sum_{s\in\mcS'}\sum_{i=1}^{n_s '} \mathcal{D}\big( g_s[\bm y_{si}^{(K)} (\mcW, \mcH);\btheta_s], \bm t_{si} \big), \label{eq:loss_fn}
\end{align}
where $\mcS'$ denotes a set of labeled node types, $g_s: \mathbb{R}^{d_s} \rightarrow \mathbb{R}^{c_s}$ denotes a node-wise function parameterized by $\btheta_s$, $\Theta \triangleq \{\btheta_s\}_{s\in\mcS'}$ denotes a set of parameters of $g_s$, and $\bm t_{si} \in \mathbb{R}^{c_s}$ is a ground-truth node-wise target.\footnote{Note that w.l.o.g., we assume that only the first $n_{s}' < n_{s}$ type-$s$ nodes have labels for training.} 
% Note that $g_s$ could be a linear transformation, MLP, or even an identity mapping. 
% Moreover, , and $\mathcal{D}$ stands for a discriminator function, \eg, cross-entropy for classification task . 
Finally, we optimize  (\ref{eq:loss_fn}) over $\mcW$, $\mcH$, and $\Theta$ via typical iterative optimization method such as stochastic gradient descent (SGD).
% \tsm{and $\Theta$ via typical iterative optimization method such as stochastic gradient descent (SGD). }

We dub the model obtained by the above procedure as HALO for \textit{\textbf{H}eterogeneous \textbf{A}rchitecture \textbf{L}everaging \textbf{O}ptimization}, and Algorithm \ref{alg:Unfolding} summarizes the overall process.
% In this learning procedure, $\bm y_{si}^* (\mbW, \mcH)$ can be considered as an trainable intermediate feature regularized via the graph structure.

% Here, we summarize our method in Algorithm \ref{alg:Unfolding}.

\begin{algorithm}[H]
\caption{HALO algorithm 
% \david{Here we could put the algorithm name, e.g., HaloGNN or HuGNN, etc.}
}\label{alg}
\begin{algorithmic}
\REQUIRE $\mcG$: Heterogeneous graph dataset
% : Sequential training datasets, 
\REQUIRE $\lambda$, $\alpha$: Hyperparamters, $K$ : Number of unfolding steps, $E$: Number of epochs
\STATE Randomly initialize $\mbH_t \in \mathcal{H}$, $\mbW_s \in \mcW$, and $\btheta_s \in \Theta$
\FOR{$e=1, \cdots, E$}
\STATE Set $\mbY_{s}^{(0)}=f(\mbX_s;\mbW_s)$ for all $s\in\mcS$
\FOR{$k=0, \cdots, K-1$}
\STATE Compute  (\ref{eq:before_prox}) using (\ref{eq:unfolding_precond}), and  (\ref{eq:after_prox}) to get $\mbY_s^{(k+1)}$ for each $s \in \mcS$ 
% \david{Should this mention equation 12?}
\ENDFOR
\STATE Compute $\ell_{\Theta}(\Theta, \mcW, \mcH)$ using $\mbY_s^{(K)}$
\STATE Optimize all $\mbW_s \in \mcW, \mbH_t \in \mcH, $ and $\btheta_s \in \Theta$ using SGD. 
% \david{There is no mention of how these are updated, i.e., SGD, etc.}
\ENDFOR
\end{algorithmic}
\label{alg:Unfolding}
\end{algorithm}

% \subsection{Computational Complexity}
% Discussion of computational complexity, especially relative to comparable GNN models like RGCN, ZooBP, etc.

% \section{Relationship with Existing Heterogeneous Methods}
% \begin{itemize}
%     \item RGCN, may need additional assumptions, for example, $\mbH_t \mbH_t^\top = \mbI$ or related.  May only be possible for first layer, but even that would be ok.  This complements existing work showing that models like TWIRLS produce GCN as a special case, at least for the first layer.
%     \item ZooBP, provide assumptions (on step size choices, etc.) such that our approach, with fixed $\mbH_t$, would reduce to ZooBP.
% \end{itemize}

% \input{Experiment}

\vspace{-.2cm}
\section{Experiments}
\vspace{-.2cm}
We evaluate HALO on node or entity classification tasks from the heterogeneous graph benchmark (HGB) \cite{(HGB)Lv2021} as well as several knowledge graph datasets \cite{(R-GCN)Schlichtkrull2018}, and compare with state-of-the-art baselines. Furthermore, we later evaluate w.r.t.~ZooBP under settings that permit direct comparisons, and we carry out additional empirical analyses and ablations.

% \vspace{-.2cm}
\subsection{Node Classification on Benchmark Datasets}
\vspace{-.2cm}
The dataset descriptions are as follows:

\noindent\textbf{HGB} contains 4 node classification datasets, which is our focus herein. These include: DBLP, IMDB, ACM, and Freebase. DBLP and ACM are citation networks, IMDB is a movie information graph, and Freebase is a large knowledge graph. Except for Freebase, all datasets have node features (for Freebase, we set the node types as the node features).
% \david{For Freebase with no node features, do we learn node embeddings, use the node type as a feature, or do something else?} 
For more details, please refer to \cite{(HGB)Lv2021}.

The \noindent\textbf{knowledge graph} benchmarking
% \david{Knowledge graph is a general concept, not just the name of a specific dataset.  Is there a more specific name?} 
proposed in \cite{(R-GCN)Schlichtkrull2018} is composed of 4 datasets:
% \david{The previous word ``datoasets" was not spelled incorrectly in the submission, so just want to check and make sure there is no problem with version control here.} 
AIFB, MUTAG, BGS, AM. The task is to classify the entities of target nodes. AIFB and MUTAG are relatively small-scale knowledge graphs, while BGS and AM are larger scale, \eg, the AM dataset has more than 1M entities. For more details, please refer to \cite{(R-GCN)Schlichtkrull2018}.

In every experiment, we chose $f(\mbX_s;\mbW_s)$ and $g_s(\bm y_s, \btheta_s)$ as linear functions for all $s \in \mcS$, and when the dimension $d_{s0}$ is different across node types, $f(\mbX_s;\mbW_s)$ is naturally constructed to align the dimensions of each feature. All models and experiments were implemented using PyTorch \cite{(Pytorch)2019} and the Deep Graph Library (DGL) \cite{(DGL)2019}. For the details on the hyperparameters and other experimental settings, please see the Supplementary Materials.

\begin{table*}[t]
\small
\centering
\caption{Results on HGB (left) and knowledge graphs (right). The results are averaged over 5 runs.}
\resizebox{1.0\linewidth}{!}{\begin{tabular}{|c|cccc|c|}
\Xhline{2\arrayrulewidth}
\Xhline{0.1pt}
Dataset      & DBLP  & IMDB  & ACM   & Freebase & Avg.  \\ \hline\hline
Metric     & \multicolumn{5}{c|}{Accuracy (\%)} \\ \hline\hline
R-GCN\cite{(R-GCN)Schlichtkrull2018}      & 92.07 & 62.05 & 91.41 & 58.33    & 75.97 \\
HAN\cite{(HAN)Wang2019}        & 92.05 & 64.63 & 90.79 & 54.77    & 75.56 \\
HGT\cite{(HGT)Hu2020}        & 93.49 & 67.20 & 91.00 & 60.51    & 78.05 \\
Simple-HGN\cite{(HGB)Lv2021} & 94.46 & 67.36 & 93.35 & \textbf{66.29}    & 80.37 \\ \hline
HALO       & \textbf{96.30} & \textbf{76.20} & \textbf{94.33} & 62.06    & \textbf{82.22} \\ 
\Xhline{2\arrayrulewidth}
\Xhline{0.1pt}
\end{tabular}
\hspace{.1in}
\begin{tabular}{|c|cccc|c|}
\Xhline{2\arrayrulewidth}
\Xhline{0.1pt}
Dataset   & AIFB  & MUTAG & BGS   & AM    & Avg.  \\ \hline\hline
Metric     & \multicolumn{5}{c|}{Accuracy (\%)} \\ \hline\hline
Feat\cite{(Feat)Paulheim2012}    & 55.55 & 77.94 & 72.41 & 66.66 & 68.14 \\
WL\cite{(WL1)Shervashidze2011,(WL2)DeVries2015}      & 80.55 & 80.88 & 86.20 & 87.37 & 83.75 \\
RDF2Vec\cite{(RDF2Vec)Ristoski2016} & 88.88 & 67.20 & 87.24 & 88.33 & 82.91 \\
R-GCN\cite{(R-GCN)Schlichtkrull2018}   & 95.83 & 73.33 & 83.10 & 89.29 & 85.39 \\ \hline
HALO    & \textbf{96.11} & \textbf{86.17} & \textbf{93.10} & \textbf{90.20} & \textbf{91.40} \\ 
\Xhline{2\arrayrulewidth}
\Xhline{0.1pt}
\end{tabular}}\label{table:Result}
\vspace{-.12in}
\end{table*}

Table \ref{table:Result} shows the results on node classification tasks, where the best performance is highlighted in bold. The left side of the table shows the results on HGB, including baselines R-GCN \cite{(R-GCN)Schlichtkrull2018}, HAN \cite{(HAN)Wang2019}, HGT \cite{(HGT)Hu2020}, and Simple-HGN \cite{(HGB)Lv2021}.
% \footnote{To keep our work to be anonymous, we did not submit the reulsts to HGB leadermoard \cite{(HGB)Lv2021}, and we instead report the validation results.\david{We should not keep this footnote.  Either we submit to the leaderboard or we just remove this footnote, but definitely can't keep it like this.}}
The right side shows the analogous results on the knowledge graph datasets, where we also include comparisons with  R-GCN \cite{(R-GCN)Schlichtkrull2018}, hand-designed feature extractors (Feat) \cite{(Feat)Paulheim2012}, Weisfeiler-Lehman kernels (WL), and RDF2Vec embeddings \cite{(RDF2Vec)Ristoski2016}.
The Avg.~column in each table represents the average of the results across the four datasets. In the table, HALO achieves the highest average accuracy, and outperforms the baselines on 7 of 8 datasets. 

\vspace{-.2cm}
\subsection{Comparison with ZooBP}
\vspace{-.2cm}
We next compare HALO with ZooBP \cite{(ZooBP)Eswaran2017}, which provided motivation for our use of compatibility matrices and represents a natural baseline to evaluate against when possible. ZooBP is a belief propagation method for heterogeneous graphs, and therefore, the number of classes per node type should be pre-defined for every node. Consequently, running ZooBP on HGB or the knowledge graph datasets in which only a single node type has labels (and therefore defined classes), \textit{e.g.}, DBLP,  is not directly feasible.
% \david{The previous description was not clear so I modified a bit; please check if the meaning is correct.} 
To that end, to compare with ZooBP, we modified the DBLP and Academic \cite{(Academic)Tang2008} datasets so that all node types have labels. Namely, in both graphs, we removed the ``Term'' and ``Venue'' node types and only utilized the ``Paper'' and ``Author'' node types. Moreover, since the original datasets only have the class labels for the ``Author'' type, we set the class labels for the ``Paper'' type as the corresponding ``Venue'' of the paper. 
% and set the labels of ``Paper'' as the index of ``Venue'' nodes. 
Thus, the 20 and 18 Venues (specified in the Supplementary Materials) become the class labels for the ``Paper'' nodes in DBLP and Academic, respectively. Furthermore, we also considered a simpler setting in which the number of class labels for the ``Paper'' nodes (denoted as $k$) is four, 
% the ``Paper'' nodes have only 4 class labels, 
by using the 4 categories (\textit{i.e., }\texttt{ML, DB, DM, IR}) of the Venues as the labels for the nodes. We note the number of class labels for the ``Author'' nodes was always set to 4 (corresponding to the Venue categories), and the resulting datasets are denoted as DBLP-reduced and Academic-reduced, respectively.\footnote{The graph in DBLP-reduced is a bipartite graph since the edges exist only between ``Paper'' and ``Author'' nodes, while the graph in Academic-reduced also has edges between the ``Paper'' nodes. }

For choosing the \textit{fixed} compatibility matrix required by ZooBP (as a label propagation-like method, ZooBP has no mechanism for learning), when $k=4$, we set $\mbH_t = \mbI_{k \times k} - \frac{1}{4} \mathbf{1}_{k \times k}$, where $\mathbf{1}_k$ is a $k \times k$ square matrix in which all elements are 1. When $k=18$ or $k=20$, we set the matrix considering the correlation between the ``Author'' and ``Paper'' labels. For example, for the ``\texttt{AAAI}'' Paper class, we set the matrix element value corresponding to the ``\texttt{ML}'' Author class higher than other Author classes, since ``\texttt{AAAI}'' venue and ``\texttt{ML}'' category are highly correlated. 
% we set the matrix element value corresponding to (``AAAI'' (Paper), ``ML'' (Author)) higher than the other elements 
% For example, "AAAI`` venue and "ML`` category are highly correlated. So, we give high positive value on ("AAAI``, "ML``). 
For further details about reducing the number of classes for the ``Paper'' nodes, setting the compatibility matrices, and the exact compatibility matrices that we used, please refer to the Supplementary Materials.
% the statistics between the labels of Authors and Papers.
% \david{The previous sentence is not clear.  What does it mean to set the matrix regarding the categories of each paper?  And I thought ZooBP sets the compatibility matrices using the label distribution frequencies between different node types.} \hong{First, I tried to set the compatibility matrix using the frequencies between node types. But it didn't work well. So, regarding the categories of each venue, I manually set the compatibility matrices. Note that the categories correspond to the labels of Papers. } \david{Well I guess my question is, which method did you use for the numbers reported in our paper?  That is what we need to describe here; can be short and simple.}
Moreover, for setting the output layer $g_s(\bm y_s, \btheta_s)$ in our case, we set $g_s(\bm y_s, \btheta_s)$ as an identity function, \ie, $\btheta_s = \mbI$, and we treat $\bm y_s$ as a score vector on each class.
% For further details on reducing the number of classes and choosing the compatibility matrix, please see the Supplementary Materials.

Table \ref{table:ZooBP} shows the results of ZooBP and HALO on the DBLP-bipartite and Academic datasets as described above. We also report the itemized classification accuracy on ``Author'' and ``Paper'' node types. From these results, we observe that HALO outperforms ZooBP in all settings. For large $k$, the gap is generally more significant. The main reason for this phenomenon is that while HALO utilizes both node features and trainable compatibility matrices, ZooBP only uses training labels with fixed compatibility matrices. To the best of our knowledge, HALO is the only method designed to exploit a trainable compatibility matrix.

\begin{table*}[t]
\small
\centering
\caption{Comparison with ZooBP}
\resizebox{1.0\linewidth}{!}{\begin{tabular}{|c|cc|cc|}
\Xhline{2\arrayrulewidth}
\Xhline{0.1pt}
Dataset    & \multicolumn{2}{c|}{DBLP-reduced}                                & \multicolumn{2}{c|}{Academic-reduced}                                       \\ \hline\hline
Metric    & \multicolumn{4}{c|}{Accuracy (\%)}                      \\ \hline\hline
Node types & \multicolumn{2}{c|}{Author / Paper / All}                          & \multicolumn{2}{c|}{Author / Paper / All}                           \\ \hline\hline
\# Classes of ``Paper'' & \multicolumn{1}{c|}{$k=4$}                 & $k=20$                & \multicolumn{1}{c|}{$k=4$}                 & $k=18$                \\ \hline
ZooBP\cite{(ZooBP)Eswaran2017}      & \multicolumn{1}{c|}{63.00 / 63.60 / 63..47} & 62.00 / 19.53 / 28.47 & \multicolumn{1}{c|}{79.63 / 73.43 / 75.97} & 81.76 / 16.08 / 42.94 \\
Ours       & \multicolumn{1}{c|}{\textbf{76.50 / 64.13 / 66.79}} & \textbf{78.50 / 33.33 / 43.00} & \multicolumn{1}{c|}{\textbf{96.63 / 93.89 / 94.93}} & \textbf{97.81 / 84.73 / 90.18} \\ 
\Xhline{2\arrayrulewidth}
\Xhline{0.1pt}
\end{tabular}}\label{table:ZooBP}
\vspace{-.15in}
\end{table*}

\begin{table}[h]
    \vspace{-.1in}
    \begin{minipage}[b]{0.34\linewidth}
        \centering
        \resizebox{\linewidth}{!} {
        % \begin{tabular}{|ccc|}
        % \Xhline{2\arrayrulewidth}
        % \Xhline{0.1pt}
        %   &  Fixed  & $\mbH_t$   \\
        % Dataset & $\mbH_t=\mbI$ & trained \\
        % \hline\hline
        % DBLP     & 95.40           & \textbf{96.30} \\
        % MUTAG    & 77.94           & \textbf{86.17} \\
        % BGS      & 86.90           & \textbf{93.10} \\ 
        % \Xhline{2\arrayrulewidth}
        % \Xhline{0.1pt}
        % \end{tabular}
        \begin{tabular}{|c|ccc|}
        \Xhline{2\arrayrulewidth}
        \Xhline{0.1pt}
                &                & Fixed         & HALO    \\
        Methods & HALO           & $\mbH_t=\mbI$ & w/o PGD \\ \hline\hline
        Metric  & \multicolumn{3}{c|}{Accuracy (\%)}      \\ \hline\hline
        DBLP    & \textbf{96.30} & 95.40         & 95.90   \\ 
        MUTAG   & \textbf{86.17} & 77.94         & 80.88   \\ 
        BGS     & \textbf{93.10} & 86.90         & 86.17   \\ 
        \Xhline{2\arrayrulewidth}
        \Xhline{0.1pt}
        \end{tabular}
        }
        \vspace{0.37in}
        \caption{\small Ablation study}
        \vspace{0.17in}
        \label{table:Ablation Study}
    \end{minipage}
    \hfill
    \begin{minipage}[b]{0.64\linewidth}
        % \centering
        \includegraphics[width=1.0\textwidth]{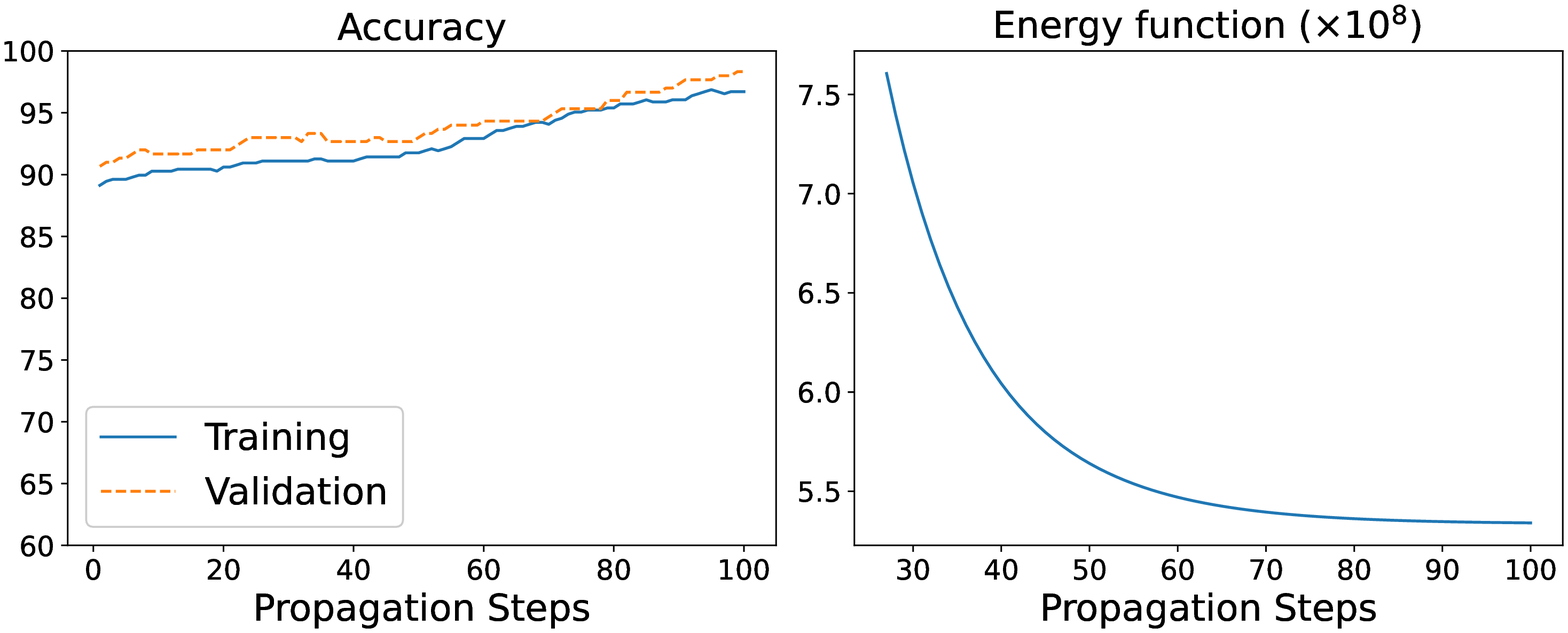}
        \vspace{-.15in}
        \captionof{figure}{ Accuracy and energy function value versus the number of propagation steps on ACM.}
        \label{fig:Propagation}
    \end{minipage}
\vspace{-.3in}
\end{table}

\subsection{Further HALO Analysis}
In this section, we first carry out an ablation study to assess the importance of learnable compatibility matrices $\mbH_t \in \mcH$, as well as the value of non-linear graph propagation operators instantiated through PGD.  Later, we also show the accuracy and energy values of HALO versus the number of propagation steps to demonstrate convergence of the lower-level HALO optimization process, which naturally mitigates oversmoothing.

\noindent\textbf{Ablation Study} \ \
To examine the effectiveness of trainable compatibility matrices, we enforce $\mbH_t=\mbI$ for all $t$, using the same embedding dimension across all node types.  Other training conditions remain unchanged.  Table \ref{table:Ablation Study} compares the resulting accuracy on DBLP, MUTAG, and BGS datasets, where we observe that HALO performs significantly better, especially on MUTAG and BGS.  These results show that using trainable compatibility matrices to quantify ambiguous relationships between neighboring nodes can be an effective strategy when performing node classification with heterogeneous graphs. Note that even when the embedding dimensions are same (as in these experiments), there can still be relation-specific mismatches between different node types that exhibit heterophily characteristics, which trainable compatibility matrices can help to resolve. 

Turning to the non-linear activations, we also re-run the HALO experiments with (\ref{eq:after_prox}) removed from the algorithm, meaning PGD defaults to regular SGD as no proximal step is used. As shown in Table \ref{table:Ablation Study}, HALO without PGD performs considerably worse, implying that the inclusion of the additional non-linear activation leads to better generalization.

\noindent\textbf{Results Varying Propagation Steps} \ \
Figure \ref{fig:Propagation} shows the results of HALO with various numbers of propagation steps. Specifically, the left plot shows the training and validation accuracy on the ACM dataset, and the right plot shows the energy function value (\ref{eq:energy_matrix}) during the unfolding steps. Notably, when we increase the number of propagation steps with a suitable $\alpha$, both training and validation accuracies are gradually increasing, and the energy function value steadily decreasing as expected. These results suggest that as the gradient descent steps gradually decrease the low-level energy function, the corresponding node embeddings $\mbY_s$ that emerge are effective in performing the downstream classification task. We have also repeated this experiment using the AIFB dataset and the overall trend is same; for more details please refer to the Supplementary Materials.
% \david{Why is there no mention here of the new results related to Figure 1 that were conducted for the rebuttal and the reviewer specifically asked for?  Definitely need to mention this, and then at least add to the supplementary.}

\vspace{-.2cm}
\section{Conclusion}
\vspace{-.2cm}
In this paper, we have proposed HALO, a novel HGNN architecture derived from the minimization steps of a relation-aware energy function, the latter consisting of trainable relation-dependent compatibility matrices that can resolve the mismatch between two different node types. Because of intrinsic properties of this unfolding framework, HALO is naturally robust to oversmoothing problems, and outperforms SOTA models on various benchmark datasets. For future work, extending HALO to spatio-temporal graphs in which the node and edge information change continuously can be considered.

% \david{I think the supplementary still needs some work.  Note that we promised various things to the reviewers in the rebuttal, and this must be added to the supplementary as we discussed, e.g., the extra figures related to energy descent, etc.  Additionally, the theorem/lemma statements have been refined, so we need to update them in the supplementary as well for consistency, and also make sure the proofs are complete.}

\section*{Acknowledgments}

This work was supported in part by the New Faculty Startup Fund from Seoul National University, NRF grants [NRF-2021R1A2C2007884, NRF-2021M3E5D2A01024795] and IITP grants [No.2021-0-01343, No.2021-0-02068, No.2022-0-00959, No.2022-0-00113] funded by the Korean government, and SNU-NAVER Hyperscale AI Center.

% \input{PaperContents}

% \david{If possible, please format the references consistently.  Now there is considerable variance.}

\bibliographystyle{plain}
\bibliography{bibfile,wipf_refs}
\newpage

%%%%%%%%%%%%%%%%%%%%%%%%%%%%%%%%%%%%%%%%%%%%%%%%%%%%%%%%%%%%
\section*{Checklist}

\begin{enumerate}

\item For all authors...
\begin{enumerate}
  \item Do the main claims made in the abstract and introduction accurately reflect the paper's contributions and scope?
    \answerYes{}
  \item Did you describe the limitations of your work?
    \answerYes{Please see Supplementary Materials.} 
  \item Did you discuss any potential negative societal impacts of your work?
    \answerYes{Please see Supplementary Materials.} 
  \item Have you read the ethics review guidelines and ensured that your paper conforms to them?
    \answerYes{}
\end{enumerate}

\item If you are including theoretical results...
\begin{enumerate}
  \item Did you state the full set of assumptions of all theoretical results?
    \answerYes{Please see Supplementary Materials.} 
        \item Did you include complete proofs of all theoretical results?
    \answerYes{Please see Supplementary Materials.} 
\end{enumerate}

\item If you ran experiments...
\begin{enumerate}
  \item Did you include the code, data, and instructions needed to reproduce the main experimental results (either in the supplemental material or as a URL)?
    \answerNo{Codes can be submitted after acceptance.} 
  \item Did you specify all the training details (e.g., data splits, hyperparameters, how they were chosen)?
    \answerYes{Please see Supplementary Materials.} 
        \item Did you report error bars (e.g., with respect to the random seed after running experiments multiple times)?
    \answerYes{Please see Supplementary Materials.} 
        \item Did you include the total amount of compute and the type of resources used (e.g., type of GPUs, internal cluster, or cloud provider)?
    \answerYes{Please see Supplementary Materials.} 
\end{enumerate}

\item If you are using existing assets (e.g., code, data, models) or curating/releasing new assets...
\begin{enumerate}
  \item If your work uses existing assets, did you cite the creators?
    \answerYes{}
  \item Did you mention the license of the assets?
    \answerNA{}
  \item Did you include any new assets either in the supplemental material or as a URL?
    \answerNo{}
  \item Did you discuss whether and how consent was obtained from people whose data you're using/curating?
    \answerNA{}
  \item Did you discuss whether the data you are using/curating contains personally identifiable information or offensive content?
    \answerNA{}
\end{enumerate}

\item If you used crowdsourcing or conducted research with human subjects...
\begin{enumerate}
  \item Did you include the full text of instructions given to participants and screenshots, if applicable?
    \answerNA{}
  \item Did you describe any potential participant risks, with links to Institutional Review Board (IRB) approvals, if applicable?
    \answerNA{}
  \item Did you include the estimated hourly wage paid to participants and the total amount spent on participant compensation?
    \answerNA{}
\end{enumerate}

\end{enumerate}

%%%%%%%%%%%%%%%%%%%%%%%%%%%%%%%%%%%%%%%%%%%%%%%%%%%%%%%%%%%%

% \david{}

\newpage
\section{Proof of Lemma \ref{lemma:closed-form}}

Before proceeding the proof of Lemma \ref{lemma:closed-form}, we first provide a basic mathematical result.

\begin{lemma}\label{lemma:RothColumn-supp} (Roth's Column Lemma \cite{(RothColumn)Harold1981}). For any three matrices $\mbX, \mbY$ and $\mbZ$,
\begin{align}
    vec(\mbX\mbY\mbZ) = (\mbZ^{\top}\otimes\mbX) vec(\mbY) \label{eq:RothColumn-supp}
\end{align}
\end{lemma}

We now proceed with the proof of our result.

\begin{proof}
The gradient of (\ref{eq:energy_matrix}) is as follows:
\begin{align}
    \nabla_{\mbY_s}\ell_{\mbY}(\mbY) = (\mbI +\lambda\mbD_s)\mbY_s-f(\mbX_s;\mbW_s) + \lambda \sum_{s'\in\mcS}\sum_{t\in\mcT_{ss'}} \big(\mbD_{st}\mbY_s(\mbH_t\mbH_t^{\top})-\mbA_t\mbY_{s'}(\mbH_t^{\top}+\mbH_{t_{\text{inv}}})\big).\label{eq:gradient-supp}
\end{align}

By applying a $vec(\cdot)$ operation to both sides of (\ref{eq:gradient-supp}) we obtain:

\begin{align}
    vec(\nabla_{\mbY_s}\ell(\mbY)) &= vec(\mbY_s)-vec(f(\mbX_s;\mbW_s)) \nonumber\\
    &+ \lambda (\sum_{s'\in\mcS}\sum_{t\in\mcT_{ss'}} (vec(\mbD_{st}\mbY_s(\mbH_t\mbH_t^{\top}))-vec(\mbA_t\mbY_{s'}(\mbH_t^{\top}+\mbH_{t_{\text{inv}}})))+vec(\mbD_s\mbY_s)) \label{eq:gradient_vec-supp}
\end{align}

Here, using Roth's Column Lemma \ref{lemma:RothColumn-supp}, we define the matrices $\mbP, \mbQ, \mbD$ as follows:

\begin{align}
    \mbP = \left[ \begin{array}{ccc}
        \mbP_{11} & ... & \mbP_{1|\mcS|}  \\
        ... & ... & ...  \\
        \mbP_{|\mcS|1} & ... & \mbP_{|\mcS||\mcS|} 
    \end{array} \right] ; \mbP_{ss'} = \sum_{t\in\mcT_{ss'}} ((\mbH_t^{\top}+\mbH_{t_{\text{inv}}}) \otimes \mbA_t) \nonumber
\end{align}

\begin{align}
    \mbQ = \bigoplus_{s\in\mcS} \mbQ_s; \mbQ_s=\sum_{s'\in\mcS}\sum_{t\in\mcT_{ss'}} (\mbH_t\mbH_t^{\top}\otimes\mbD_{st}), \mbD = \bigoplus_{s\in\mcS}\mbI\otimes\mbD_s. \nonumber
\end{align}
Therefore, after rewriting (\ref{eq:gradient_vec-supp}) as

\begin{align}
    vec(\nabla_{\mbY_s}\ell_{\mbY}(\mbY)) &= vec(\mbY_s) - vec(f(\mbX_s;\mbW_s)) \nonumber\\
    &+ \lambda\big(\mbQ_s vec(\mbY_s) - \sum_{s'\in\mcS}\mbP_{ss'}vec(\mbY_{s'})+\mbD_s vec(\mbY_s)\big), \label{eq:gradient_PQD_s-supp}
\end{align}
we can stack $|\mcS|$ such matrix equations together using $\mbP, \mbQ, \mbD$ to obtain:

\begin{align}
    vec(\nabla_{\mbY}\ell_{\mbY}(\mbY)) &= vec(\mbY) - vec(\widetilde{f}(\widetilde{\mbX};\mcW)) + \lambda\big(\mbQ -\mbP+\mbD \big) vec(\mbY). \label{eq:gradient_stack-supp}
\end{align}

Since $\ell_{\mbY}(\mbY)$ is a convex function, a point that achieves $\nabla_{\mbY} \ell_{\mbY}(\mbY)=0$ is a optimal point. Therefore, the closed-form solution for $vec(\mbY^* (\mcW, \mcH))$ is derived as:

\begin{align}
    vec(\mbY^* (\mcW, \mcH)) = (\mbI + \lambda(\mbQ-\mbP+\mbD))^{-1}vec(\widetilde{f}(\widetilde{\mbX};\mcW)). \label{eq:closed_form-supp}
\end{align}
\end{proof}

\section{Proof of Theorem (3.2, Manuscript)}
\begin{theorem} \label{theorem:convergence}
The iterations (\ref{eq:unfolding_precond}) are guaranteed to monotonically converge to the unique global minimum of (\ref{eq:energy}) provided that
\begin{align}
    \alpha < \frac{2 + 2\lambda d_{\text{min}}}{1 + \lambda (d_{\text{min}} + \sigma_{\text{max}})}, \label{eq:convergence}
\end{align}
where $d_{\text{min}}$ is the minimum diagonal element of $\mbD$ and $\sigma_{\text{max}}$ is a maximum eigenvalue of matrix $(\mbQ - \mbP)$.
\end{theorem}

\begin{proof}
The energy function and iteration $k+1$ of gradient descent on $\mbY_s$ with preconditioning is as follows:

\begin{align}
    \ell_{\mbY}(\mbY) \triangleq & \sum_{s\in\mcS}\left[\frac{1}{2}||\mbY_s-f(\mbX_s;\mbW_s)||_{\mathcal{F}}^2 + \frac{\lambda}{2}\sum_{s'\in\mcS}\sum_{t\in\mcT_{ss'}}\sum_{(i,j)\in\mcE_t}||\bm y_{si}\mbH_t-\bm y_{s'j}||_2^2 \right] \label{eq:energy-supp}
\end{align}

\begin{align}
    \mbY_s^{(k+1)} = \mbY_s^{(k)} - \alpha \widetilde{\mbD}_s^{-1} \nabla_{\mbY_s^{(k)}}\ell_{\mbY}(\mbY). \label{eq:unfolding_precond-supp}
\end{align}
By applying the $vec(\cdot)$ operation to both sides of \ref{eq:unfolding_precond-supp}, this can be transformed into:
\begin{align}
    vec(\mbY_s^{(k+1)}) &= vec(\mbY_s^{(k)}) - \alpha vec( \widetilde{\mbD}_s^{-1} \nabla_{\mbY_s^{(k)}}\ell_{\mbY}(\mbY)) \label{eq:vec_unfolding_precond-supp} \\
    &= vec(\mbY_s^{(k)}) - \alpha (\mbI \otimes \widetilde{\mbD}_s^{-1}) vec(\nabla_{\mbY_s^{(k)}}\ell_{\mbY}(\mbY)). \label{eq:vec_unfolding_precond_roth-supp}
\end{align}

Note that we apply Roth's column lemma to (\ref{eq:vec_unfolding_precond-supp}) to derive (\ref{eq:vec_unfolding_precond_roth-supp}). Stacking $|\mcS|$ such vectors together, (\ref{eq:vec_unfolding_precond_roth-supp}) can be written as:
\begin{align}
    vec(\mbY^{(k+1)}) = vec(\mbY^{(k)}) - \alpha \widetilde{\mbD}^{-1} vec(\nabla_{\mbY^{(k)}}\ell_{\mbY}(\mbY)), \label{eq:stack_vec_unfolding_precond_roth-supp}
\end{align}
where $\widetilde{\mbD}^{-1} \triangleq \bigoplus_{s \in \mcS} \mbI \otimes \widetilde{\mbD}_s^{-1}$.

Because $\ell_{\mbY} (\mbY)$ is convex, for any $\mbY^{(k+1)}$ and $\mbY^{(k)}$ the following inequality holds:
\begin{align}
    \ell_{\mbY} (\mbY^{(k+1)}) \leq \ell_{\mbY} (\mbY^{(k)}) &+ vec(\nabla_{\mbY^{(k)}}\ell_{\mbY}(\mbY))^{\top} vec(\mbY^{(k+1)} - \mbY^{(k)}) \nonumber\\
    &+ \frac{1}{2} vec(\mbY^{(k+1)} - \mbY^{(k)})^{\top} \nabla_{\mbY^{(k)}}^2 \ell_{\mbY}(\mbY) vec(\mbY^{(k+1)} - \mbY^{(k)}), \label{ineq:second_order}
\end{align}
where $\nabla^2_{\mbY^{(k)}} \ell_{\mbY}(\mbY)$ is a Hessian matrix whose elements are $\nabla^2_{\mbY^{(k)}} \ell_{\mbY}(\mbY)_{ij}=\frac{\partial \ell _{\mbY}(\mbY)}{\partial vec(\mbY)_i \partial vec(\mbY)_j} |_{\mbY=\mbY^{(k)}}$.

Plugging in the gradient descent update by letting $vec(\mbY^{(k+1)} - \mbY^{(k)}) = -\alpha \widetilde{\mbD}^{-1} vec(\nabla_{\mbY^{(k)}}\ell_{\mbY}(\mbY))$, we get:

\begin{align}
    \ell_{\mbY} (\mbY^{(k+1)}) \leq \ell_{\mbY} (\mbY^{(k)}) &- (\widetilde{\mbD}^{-1} vec(\nabla_{\mbY^{(k)}}\ell_{\mbY}(\mbY)))^{\top} (\alpha \widetilde{\mbD}) (\widetilde{\mbD}^{-1} vec(\nabla_{\mbY^{(k)}}\ell_{\mbY}(\mbY))) \nonumber\\
    &+(\widetilde{\mbD}^{-1} vec(\nabla_{\mbY^{(k)}}\ell_{\mbY}(\mbY)))^{\top} (\frac{\alpha^2}{2}\nabla_{\mbY^{(k)}}^2 \ell_{\mbY}(\mbY))  (\widetilde{\mbD}^{-1} vec(\nabla_{\mbY^{(k)}}\ell_{\mbY}(\mbY))).
\end{align}

If $\alpha \widetilde{\mbD} - \frac{\alpha^2}{2}\nabla^2_{\mbY^{(k)}} \ell_{\mbY}(\mbY) \succ 0$ holds, then gradient descent will never increase the loss, and moreover, since $\ell_{\mbY} (\mbY)$ is strongly convex, it will monotonically decrease the loss until the unique global minimum is obtained.  To compute $\nabla^2_{\mbY^{(k)}}\ell_{\mbY}(\mbY)$, we differentiate (\ref{eq:gradient_stack-supp}) and arrive at:
\begin{align}
    \nabla^2_{\mbY^{(k)}}\ell_{\mbY}(\mbY) = \mbI + \lambda (\mbQ - \mbP + \mbD).
\end{align}

Returning to the above inequality, we can then proceed as follows:
\begin{align}
    \alpha \widetilde{\mbD} - \frac{\alpha^2}{2}(\mbI + \lambda(\mbQ - \mbP + \mbD)) &= \alpha (\mbI+\lambda \mbD) - \frac{\alpha^2}{2}(\mbI + \lambda(\mbQ - \mbP + \mbD)) \nonumber\\
    &= (\alpha - \frac{\alpha^2}{2})(\mbI + \lambda\mbD) - \frac{\alpha^2 \lambda}{2} (\mbQ - \mbP) \nonumber\\
    &\succ (\alpha - \frac{\alpha^2}{2})(1 + \lambda d_{\text{min}}) \mbI - \frac{\alpha^2 \lambda}{2} (\mbQ - \mbP).
\end{align}
If $\alpha$ satisfies $(\alpha - \frac{\alpha^2}{2})(1 + \lambda d_{\text{min}}) \mbI - \frac{\alpha^2 \lambda}{2} (\mbQ - \mbP) \succ 0$, then $\alpha \widetilde{\mbD} - \frac{\alpha^2}{2}\nabla^2_{\mbY^{(k)}} \ell_{\mbY}(\mbY) \succ 0$ holds. Therefore, a sufficient condition for convergence to the unique global optimum is:
\begin{align}
    (\alpha - \frac{\alpha^2}{2})(1 + \lambda d_{\text{min}}) - \frac{\alpha^2 \lambda}{2} \sigma_{\text{max}} > 0.
\end{align}
Consequently, to guarantee the aforementioned convergence we arrive at the final inequality:
\begin{align}
    \alpha < \frac{2 + 2\lambda d_{\text{min}}}{1 + \lambda(d_{\text{min}}+\sigma_{\text{min}})}.
\end{align}
\end{proof}

\section{Additional Experiment results}

\begin{figure}[h]
    \centering
    \includegraphics[width=0.9\textwidth]{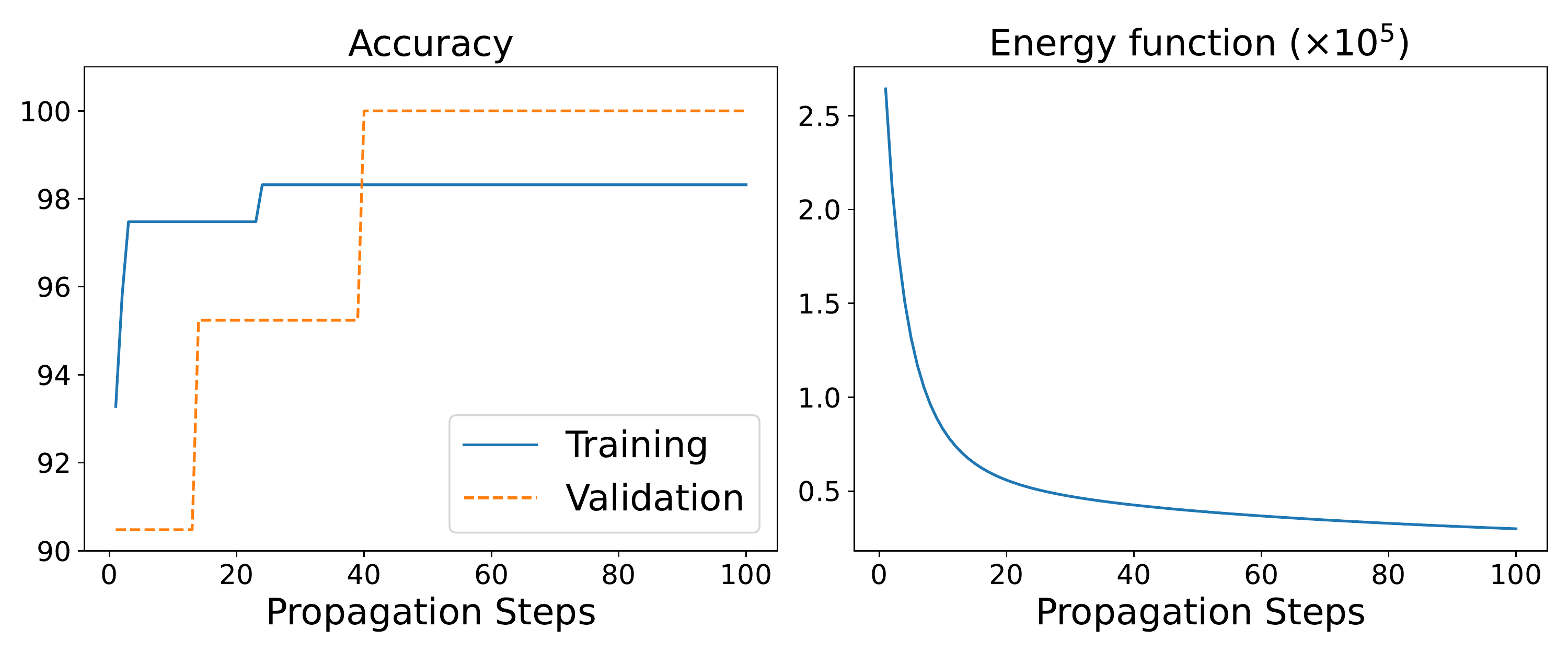}
    \vspace{-.1in}
    \caption{Accuracy and energy function value versus the number of propagation steps on AIFB.
    }\label{fig:AIFB_Prop_step}
        \vspace{-.1in}
\end{figure}
\subsection{Results Varying Propagation Steps}
Figure \ref{fig:AIFB_Prop_step} shows the results of HALO with various number of propagation steps in AIFB dataset experiment. Similar to the results on the ACM dataset in our main paper, the training and validation accuracy increase when we increase the number of propagation steps, while the energy function value is steadily decreasing.

\subsection{Results on Different Base Models and Test Time Comparison}

Table \ref{table:Additional_results} shows the results of applying HALO with different base models, as well as test time comparison between HALO and R-GCN\cite{(R-GCN)Schlichtkrull2018}. For the former, we use a 2-layer MLP $f(\mbX;\mbW_s)$ per node type. In Table \ref{table:Additional_results} (left), HALO with the MLP base model achieves comparable results to HALO using a linear layer for the base model. Of course with larger or more complex datasets, an MLP could potentially still be advantageous.

As we mentioned in the main paper, the time complexity of propagating the graph features in HALO is on the same level as R-GCN. To corroborate this claim empirically, we carry out an experiment comparing the test times of HALO and R-GCN.  In this experiment, both R-GCN and HALO used 16 GNN layers with 16 hidden dimensions. In Table \ref{table:Additional_results} (right), as expected, R-GCN and HALO take almost same amount of time when the model sizes are same. 

\begin{table*}[h]
\small
\centering
\caption{Results using different base models (left) and test time comparisons (right).}
\resizebox{1.0\linewidth}{!}{\begin{tabular}{|c|ccc|}
\Xhline{2\arrayrulewidth}
\Xhline{0.1pt}
Datasets  & MUTAG      & BGS       & DBLP      \\ \hline\hline
Metric             & \multicolumn{3}{c|}{Accuracy (\%)} \\ \hline\hline
HALO               & 86.17      & 93.10     & 96.30     \\
HALO w/ MLP $f(\mbX;\mbW_s)$ & 85.58      & 93.79     & 94.75     \\ 
\Xhline{2\arrayrulewidth}
\Xhline{0.1pt}
\end{tabular}
\hspace{.1in}
\begin{tabular}{|c|cc|}
\Xhline{2\arrayrulewidth}
\Xhline{0.1pt}
Datasets & MUTAG            & AIFB            \\ \hline\hline
Metric  & \multicolumn{2}{c|}{Time (second)} \\ \hline\hline
R-GCN   & 0.8424           & 1.8403          \\
HALO    & 0.7740           & 1.7065          \\ 
\Xhline{2\arrayrulewidth}
\Xhline{0.1pt}
\end{tabular}}\label{table:Additional_results}
\vspace{-.12in}
\end{table*}

\subsection{Results with standard error}
Table \ref{table:Result} reproduces the results from Table 1 in the main paper, but with the average accuracy and corresponding standard deviations across 5 runs included.

\begin{table}[t]
\small
\centering
\caption{Results on HGB (left) and knowledge graphs (right). The results are averaged over 5 runs, with standard deviations included.}
\resizebox{1.0\linewidth}{!}{\begin{tabular}{|c|cccc|}
\Xhline{2\arrayrulewidth}
\Xhline{0.1pt}
Dataset    & DBLP            & IMDB            & ACM             & Freebase        \\ \hline\hline
Metric     & \multicolumn{4}{c|}{Accuracy (\%)}                                    \\ \hline\hline
R-GCN      & 92.07$\pm$ 0.50 & 62.05$\pm$ 0.15 & 91.41$\pm$ 0.75 & 58.33$\pm$ 1.57 \\
HAN        & 92.05$\pm$ 0.62 & 64.63$\pm$ 0.58 & 90.79$\pm$ 0.43 & 54.77$\pm$ 1.40 \\
HGT        & 93.49$\pm$ 0.25 & 67.20$\pm$ 0.57 & 91.00$\pm$ 0.76 & 60.51$\pm$ 1.16 \\
Simple-HGN & 94.46$\pm$0.22  & 67.36$\pm$0.57  & 93.35$\pm$0.45  & \textbf{66.29$\pm$0.45}  \\ \hline
HALO       & \textbf{96.30$\pm$0.46}  & \textbf{76.20$\pm$0.77}  & \textbf{94.33$\pm$1.00}  & 62.06$\pm$0.74  \\ \hline
\end{tabular}
\hspace{.1in}
\begin{tabular}{|c|cccc|}
\Xhline{2\arrayrulewidth}
\Xhline{0.1pt}
Dataset & AIFB           & MUTAG          & BGS            & AM             \\ \hline\hline
Metric  & \multicolumn{4}{c|}{Accuracy (\%)}                                \\ \hline\hline
Feat    & 55.55$\pm$0.00 & 77.94$\pm$0.00 & 72.41$\pm$0.00 & 66.66$\pm$0.00 \\
WL      & 80.55$\pm$0.00 & 80.88$\pm$0.00 & 86.20$\pm$0.00 & 87.37$\pm$0.00 \\
RDF2Vec & 88.88$\pm$0.00 & 67.20$\pm$1.24 & 87.24$\pm$0.89 & 88.33$\pm$0.61 \\
R-GCN   & 95.83$\pm$0.62 & 73.23$\pm$0.48 & 83.10$\pm$0.80 & 89.29$\pm$0.35 \\ \hline
HALO    & \textbf{96.11$\pm$2.22} & \textbf{86.17$\pm$1.18} & \textbf{93.10$\pm$2.18} & \textbf{90.20$\pm$1.20} \\ 
\Xhline{2\arrayrulewidth}
\Xhline{0.1pt}\hline
\end{tabular}}\label{table:Result}
\vspace{-.12in}
\end{table}

\section{Details on Experiment Settings}

\begin{table}[h]
\small
\centering
\caption{Model Hyperparameters for HALO}
\resizebox{0.8\linewidth}{!}{\begin{tabular}{c|cccccccc}
\Xhline{2\arrayrulewidth}
\Xhline{0.1pt}
Datasets                                                     & DBLP      & IMDB      & ACM       & Freebase  & AIFB      & MUTAG     & BGS       & AM        \\ \hline\hline
\begin{tabular}[c]{@{}c@{}}Hidden Layer\\  Size\end{tabular} & 256       & 64        & 32        & 32        & 16        & 16        & 16        & 16        \\
Learning Rate                                                & $10^{-4}$ & $10^{-3}$ & $10^{-2}$ & $10^{-2}$ & $10^{-3}$ & $10^{-3}$ & $10^{-2}$ & $10^{-2}$ \\
Weight Decay                                                 & $10^{-5}$ & $10^{-5}$ & $10^{-4}$ & $10^{-3}$ & $10^{-5}$ & $10^{-4}$ & $10^{-5}$ & $10^{-4}$ \\
K                                                            & 8         & 32        & 32        & 4         & 16        & 16        & 8         & 4         \\
$\lambda$                                                    & 1         & 1         & 0.1       & 1         & 1         & 0.01      & 0.1       & 1         \\
$\alpha$                                                     & 1         & 1         & 0.1       & 1         & 0.1       & 1         & 1         & 1         \\ 
\Xhline{2\arrayrulewidth}
\Xhline{0.1pt}
\end{tabular}}\label{table:Hyperparameters}
\vspace{-.15in}
\end{table}

\subsection{Hyperparameters for the experimental results from Table 1 (Manuscript)}
In the all experiments, we used Adam optimizer \cite{(Adam)Kingma2015} and dropout as regularization with dropout rate 0.5. For other hyperparameters, please refer to Table \ref{table:Hyperparameters}.

\subsection{Categorization in DBLP and Academic datasets}

To carry out experiments for ZooBP, we modified DBLP and Academic datasets as mentioned in the main text. Here, we provide details on the categorization of venues in these datasets. Each venue is the name of academic conference.

\noindent\textbf{DBLP} dataset has 20 venues: AAAI, CVPR, ECML, ICML, IJCAI, SIGMOD, VLDB, EDBT, ICDE, PODS, ICDM, KDD, PAKDD, PKDD, SDM, CIKM, CIR, SIGIR, WSDM, WWW.  We categorize the above venues to 4 categories: \textbf{ML}  (AAAI, CVPR, ECML, ICML, IJCAI), \textbf{DB} (SIGMOD, VLDB, EDBT, ICDE, PODS), \textbf{DM} (ICDM, KDD, PAKDD, PKDD, SDM), and \textbf{IR}(CIKM, CIR, SIGIR, WSDM, WWW).

\noindent\textbf{Academic} dataset has 18 venues: ICML, AAAI, IJCAI, CVPR, ICCV, ECCV, ACL, EMNLP, NAACL, KDD, WSDM, ICDM, SIGMOD, VLDB, ICDE, WWW, SIGIR, CIKM. We categorize the above venues to 4 categories: \textbf{ML} (ICML, AAAI, IJCAI), \textbf{Vision} (CVPR, ICCV, ECCV), \textbf{NLP} (ACL, EMNLP, NAACL), and \textbf{Data} (KDD, WSDM, ICDM, SIGMOD, VLDB, ICDE, WWW, SIGIR, CIKM).

\begin{figure}[htb]
    \centering
    \includegraphics[width=0.9\textwidth]{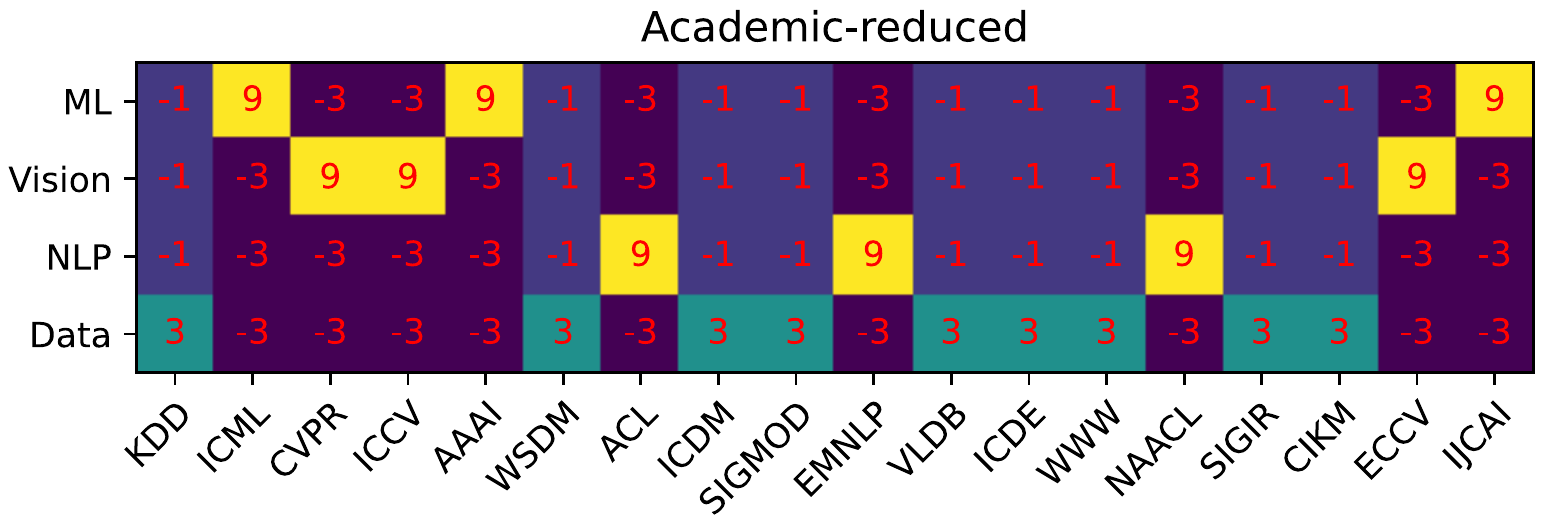}
    \vspace{-.1in}
    \caption{Compatibility matrix used in Academic-reduced datsaet}\label{fig:comp_matrix_academic}
    \centering
    \includegraphics[width=0.9\textwidth]{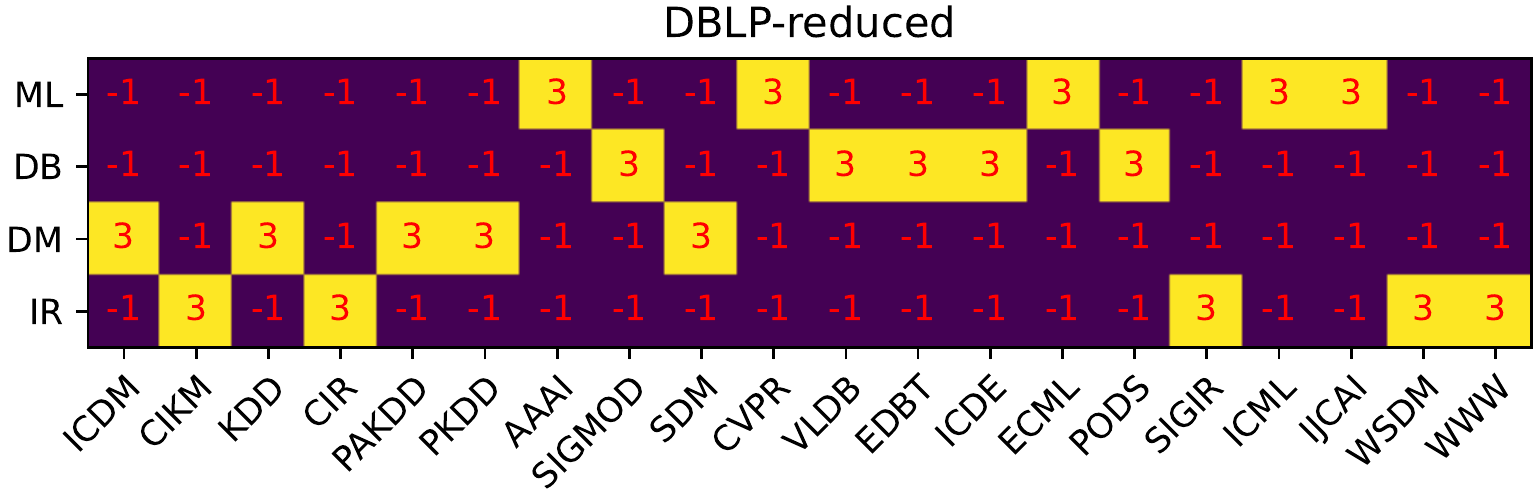}
    \vspace{-.1in}
    \caption{Compatibility matrix used in DBLP-reduced datsaet}\label{fig:comp_matrix_dblp}
        \vspace{-.1in}
\end{figure}

For choosing the compatibility matrix of ZooBP, based on the category of each venue, we give high positive value on $H(i,j)$ if $j$-th venue corresponds to $i$-th venue. For example, ``AAAI'' in DBLP dataset belongs to ``ML'' category. Therefore, we give high positive value on $H(\text{``ML''}, \text{``AAAI''})$. Otherwise, we give negative value to satisfy the residual condition of the compatibility matrix.  The results for the Academic and DBLP data are shown in Figures \ref{fig:comp_matrix_academic} and \ref{fig:comp_matrix_dblp} respectively.

\section{Limitations and Potential Negative Social Impact}
One limitation is that we have thus far not integrated HALO with large-scale sampling, which would allow us to apply it to huge graphs. And a potential negative societal impact is that as the node classification accuracy with heterogeneous graphs significantly improves with models like HALO, it could be maliciously used to improve the quality of the recommendation of socially damaging products on the Internet (e.g., dangerous weapons or harmful videos on streaming websites).

\end{document}